\PassOptionsToPackage{unicode}{hyperref}
\PassOptionsToPackage{hyphens}{url}
\PassOptionsToPackage{dvipsnames,svgnames,x11names}{xcolor}

\documentclass[12pt]{article}

\usepackage{amsmath,amssymb}
\usepackage{iftex}

\usepackage[T1]{fontenc}
\usepackage[utf8]{inputenc}
\usepackage{textcomp}

\usepackage{lmodern}
\ifPDFTeX\else  
    
\fi

\IfFileExists{upquote.sty}{\usepackage{upquote}}{}
\IfFileExists{microtype.sty}{
  \usepackage[]{microtype}
  \UseMicrotypeSet[protrusion]{basicmath} 
}{}
\makeatletter
\@ifundefined{KOMAClassName}{
  \IfFileExists{parskip.sty}{
    \usepackage{parskip}
  }{
    \setlength{\parindent}{0pt}
    \setlength{\parskip}{6pt plus 2pt minus 1pt}}
}{
  \KOMAoptions{parskip=half}}
\makeatother
\usepackage{xcolor}
\makeatletter
\ifx\paragraph\undefined\else
  \let\oldparagraph\paragraph
  \renewcommand{\paragraph}{
    \@ifstar
      \xxxParagraphStar
      \xxxParagraphNoStar
  }
  \newcommand{\xxxParagraphStar}[1]{\oldparagraph*{#1}\mbox{}}
  \newcommand{\xxxParagraphNoStar}[1]{\oldparagraph{#1}\mbox{}}
\fi
\ifx\subparagraph\undefined\else
  \let\oldsubparagraph\subparagraph
  \renewcommand{\subparagraph}{
    \@ifstar
      \xxxSubParagraphStar
      \xxxSubParagraphNoStar
  }
  \newcommand{\xxxSubParagraphStar}[1]{\oldsubparagraph*{#1}\mbox{}}
  \newcommand{\xxxSubParagraphNoStar}[1]{\oldsubparagraph{#1}\mbox{}}
\fi
\makeatother

\usepackage{longtable,booktabs,array}
\usepackage{calc} 

\usepackage{etoolbox}
\makeatletter
\patchcmd\longtable{\par}{\if@noskipsec\mbox{}\fi\par}{}{}
\makeatother

\IfFileExists{footnotehyper.sty}{\usepackage{footnotehyper}}{\usepackage{footnote}}
\makesavenoteenv{longtable}
\usepackage{graphicx}
\makeatletter
\def\maxwidth{\ifdim\Gin@nat@width>\linewidth\linewidth\else\Gin@nat@width\fi}
\def\maxheight{\ifdim\Gin@nat@height>\textheight\textheight\else\Gin@nat@height\fi}
\makeatother

\setkeys{Gin}{width=\maxwidth,height=\maxheight,keepaspectratio}

\makeatletter
\def\fps@figure{htbp}
\makeatother

\makeatletter
\@ifpackageloaded{caption}{}{\usepackage{caption}}
\AtBeginDocument{
\ifdefined\contentsname
  \renewcommand*\contentsname{Table of contents}
\else
  \newcommand\contentsname{Table of contents}
\fi
\ifdefined\listfigurename
  \renewcommand*\listfigurename{List of Figures}
\else
  \newcommand\listfigurename{List of Figures}
\fi
\ifdefined\listtablename
  \renewcommand*\listtablename{List of Tables}
\else
  \newcommand\listtablename{List of Tables}
\fi
\ifdefined\figurename
  \renewcommand*\figurename{Figure}
\else
  \newcommand\figurename{Figure}
\fi
\ifdefined\tablename
  \renewcommand*\tablename{Table}
\else
  \newcommand\tablename{Table}
\fi
}
\@ifpackageloaded{float}{}{\usepackage{float}}
\floatstyle{ruled}
\@ifundefined{c@chapter}{\newfloat{codelisting}{h}{lop}}{\newfloat{codelisting}{h}{lop}[chapter]}
\floatname{codelisting}{Listing}

\makeatother
\makeatletter
\@ifpackageloaded{caption}{}{\usepackage{caption}}
\@ifpackageloaded{subcaption}{}{\usepackage{subcaption}}
\makeatother

\ifLuaTeX
  \usepackage{selnolig}  
\fi
\usepackage[]{natbib}
\usepackage{bookmark}

\usepackage{xiaohanmacros_submission, amsthm, bbm, mathtools, algorithm, algorithmicx, enumitem, pdflscape, layout}
\usepackage[noend]{algpseudocode}

\newtheorem{definition}{Definition}
\newtheorem{theorem}{Theorem}
\newtheorem{assumption}{Assumption}
\newtheorem{lemma}{Lemma}
\newtheorem{remark}{Remark}
\newtheorem{corollary}{Corollary}
\newtheoremstyle{exampleStyle}  
 {3pt}                        
 {3pt}                        
 {\normalfont}                
 {}                           
 {\bfseries}                  
 {.}                          
 {.5em}                       
 {}                           

\theoremstyle{exampleStyle}
\newtheorem{example}{Example}
\begin{document}

\def\spacingset#1{\renewcommand{\baselinestretch}
{#1}\small\normalsize} \spacingset{1}

\raggedbottom
\allowdisplaybreaks[1]

\title{\vspace*{-.4in} {Transfer Learning for Classification under Decision Rule Drift with Application to Optimal Individualized Treatment Rule Estimation}}
\author{Xiaohan Wang$^1$, Yang Ning$^1$\\
    $^1$Department of Statistics and Data Science, Cornell University \\ \phantom{A}\\
    \texttt{\{xw547, yn265\} @ cornell.edu}\\
\date{\today}
 }

\maketitle
\thispagestyle{empty}

\bigskip
\begin{abstract}
In this paper, we extend the transfer learning classification framework from regression function-based methods to decision rules. We propose a novel methodology for modeling posterior drift through Bayes decision rules. By exploiting the geometric transformation of the Bayes decision boundary, our method reformulates the problem as a low-dimensional empirical risk minimization problem. Under mild regularity conditions, we establish the consistency of our estimators and derive the risk bounds. Moreover, we illustrate the broad applicability of our method by adapting it to the estimation of optimal individualized treatment rules. Extensive simulation studies and analyses of real-world data further demonstrate both superior performance and robustness of our approach.
\end{abstract}

\vfill

\newpage
\spacingset{1.8}

\section{Introduction} \label{section: intro}
Transfer learning refers to the statistical learning task in which the objective is to improve the prediction accuracy of the ``target domain'' using data from the ``source domain''. When the source sample size is very large and shares some characteristics with the target domain, it is beneficial to explore new methods to ``transfer'' the knowledge learned from the source domain to the target domain. Utilizing this idea, it is common for practitioners to implement transfer learning methodologies, such as chemical engineering \citep{wu2020fault}, medical imaging \citep{kim2022transfer}, and geography \citep{zhao2017transfer}. Additionally, a rich body of statistical theory and methods has emerged, addressing applications in high-dimensional linear regression \citep{li2023transfer}, generalized linear models \citep{tian2023transfer}, and graphical models \citep{li2023transfer}.

In this paper, we consider transfer learning problems under the binary classification setting. Formally, we assume that two datasets are observed, with the observations in each set being independently and identically distributed (i.i.d.). One is from the source distribution \(\{(X_{i}^P, Y_i^P)\}_{i= 1}^{n_P} \sim P \) and the other is from the target distribution \(\{(X_{i}^Q, Y_i^Q)\}_{i= 1}^{n_Q} \sim Q \), where \(X_i^P, X_i^Q \in \mathcal{X} \subset \R^d\) are the feature vectors and \(Y_i^P, Y_i^Q\in \left\{ -1, 1 \right\}\) are class labels. The goal is to construct a decision function \(f: \mathcal{X} \to \R\) from the two datasets to minimize the misclassification error \(R_Q(f)\) for the target distribution:
\begin{align}\label{eq_risk}
 R_Q(f)\coloneq Q\left(Y \neq \text{sign}(f)\right).
\end{align} 
The smallest achievable risk is called the {\it Bayes risk}, defined as 
\begin{align}\label{eq_risk_bayes}
R_Q^* \coloneq \inf\left\{R_Q(f)\mid f: \mathcal{X} \to \R \text{ measurable}\right\},
\end{align}
where the corresponding decision function is called the {\it Bayes decision function}, denoted as \(f_Q^*\). Equivalently, the Bayes decision function can be written as \(f_Q^*(x) = \text{sign}(\eta_Q(x)-1/2) \), where \(\eta_Q(x) \coloneqq Q(Y= 1 \mid X = x)\) is the regression function \citep{audibert2007fast, devroye2013probabilistic}. The corresponding Bayes decision rule can be equivalently defined via the set \(G_{f^*_Q} \coloneq \left\{x: f_Q^*(x) \geq 0 \right\} = \left\{ x: \eta_Q(x) \geq 1/2 \right\}\). While the Bayes decision function \(f_Q^*\) is not uniquely defined, the regression function $\eta_Q(x)$ and the set $G_{f^*_Q}$ are unique. Consequently, in the transfer learning setting, two viable approaches emerge for modeling the change in Bayes decision function: the conventional regression-based approach \citep{cai2021transfer,reeve2021adaptive,maity2024linear} and our proposed Bayes decision rule-based-method. The comparison of these methodologies will be discussed in greater detail in the subsequent sections.

\subsection{Related literature}
Transfer learning methods depend on understanding the difference between the source and target domains, a phenomenon known as drift. This drift can be broadly classified into two categories: approaches with structural assumptions and those without. We first outline the existing results for methods without structural assumptions, then consider structured scenarios commonly adopted in the literature.

The first strand of work focused on the properties of transfer learning without imposing any structural assumptions, yet the similarity was measured by discrepancy measures such as generalized Kolmogorov-Smirnov distance \citep{ben2006analysis,blitzer2007learning,cortes2019adaptation}. The primary goal is to quantify the cost of using the existing learning algorithm trained on the source domain under possible data drift. Not surprisingly, the error bound for the risk depends on the discrepancy measure between the source and target distributions \citep{david2010impossibility}. 

In order to improve the prediction error bounds, three common structural assumptions are typically made: covariate drift, label drift, and posterior drift.

Under the covariate drift assumption, it is assumed that \(P_{Y|X} = Q_{Y|X}\) and \(P_{X} \neq Q_{X}\) \citep{gong2016domain,kpotufe2021marginal,chen2021weighted}. Similar ideas have also been considered in estimating optimal individualized treatment rule under covariate drift \citep{liu2023augmented,wu2023transfer,wang2025transfer}. 

Under the label drift assumption, it is assumed that the distribution of the response differs across domains while the conditional distribution of \(X\) given \(Y\) is the same. That is, \(P_Y \neq Q_Y\), yet the conditional distribution of \(X\) given \(Y\) is the same \(P_{X|Y} = Q_{X|Y}\), such as \cite{zhang2013domain,lipton2018detecting,alexandari2020maximum,maity2022minimax,lee2024doubly}. 

Among these scenarios, posterior shift, where \(P_{Y|X} \ne Q_{Y|X}\) and \(P_{X} = Q_{X}\), is particularly challenging yet critical, as it captures situations where the conditional relationship between features and outcomes changes across domains. In binary classification settings, some recent advances model the difference between the source and target distributions through the regression function, leaving the Bayes decision rule unchanged \(G_{f^*_Q} = G_{f^*_P}\) \citep{cannings2020classification,hanneke2019value,cai2021transfer}. Specifically, \cite{cai2021transfer} modeled the underlying knowledge transfer through the signal strength of two regression functions \(\eta_P\) relative to \(\eta_Q\), that is:
\[C_\gamma |\eta_Q- 1/2|^\gamma \leq |\eta_P- 1/2|,\]
for some constants \(C_\gamma, \gamma\), where they developed fast, minimax-optimal, and achievable convergence rate results. The assumption \(G_{f^*_Q} = G_{f^*_P}\) can be relaxed under alternative modeling assumptions on regression functions  \citep{reeve2021adaptive,fan2023robust,maity2024linear}. For example, \cite{reeve2021adaptive} assumed the support of covariates can be segmented into a finite number of cells, and for each cell there exists a function \(g(\cdot)\) and a constant \(\gamma\) such that \(\frac{g(\eta_Q) - g(1/2)}{\eta_P -1/2 }\geq \gamma\), offering a flexible model. In \cite{fan2023robust}, the similarity is captured by restricting the difference close to the decision boundary of \(Q\) while allowing for different decision boundaries. In \cite{maity2024linear}, the relationship between the two regression functions is captured by a linear adjustment of the regression function. To the best of our knowledge, all the existing work in this strand has been focused on modeling the posterior drift through regression functions.

\subsection{Our contributions}
Despite theoretical developments addressing posterior drift through regression function-based approaches, direct modeling of the decision rule under posterior drift remains comparatively underexplored. 

To address this gap, we propose a framework that directly characterizes knowledge transfer via transformations of the Bayes decision rule, and we introduce an algorithm that utilizes support vector machines (SVM) and empirical risk minimization (ERM) to obtain a classifier with accompanying theoretical guarantees.

Compared to the regression function-based method, our model is more flexible since we directly impose assumptions on the Bayes decision boundary rather than the regression function. We also allow noise on the transformations of the Bayes decision rule,  offering enhanced robustness in our model specification. Specifically, we assume that there exist a known function \(h(\cdot, \cdot): \mathcal{E}\times \Theta \to \mathcal{E}\) and an unknown parameter $\theta^*\in\Theta \subseteq \R^p$ such that:
\begin{align}\label{eq: transformequation}
 d_{H}(G_{f_Q^*}, h(G_{f_P^*}, \theta^*))\vee d_{H}\left( \overline{G_{f_Q^*}^C}, \overline{h(G_{f_P^*}, \theta^*)^C} \right)\leq \delta,
\end{align}
where \(d_{H}\) is the Hausdorff distance, \(h(\cdot, \cdot)\) is a pre-specified parametric transformation of the decision rule, such as translation and rotation and $\mathcal{E}$ denotes the collection of all possible sets in $\R^d$. The unknown parameter \(\theta^*\) represents the ``strength'' of knowledge transfer, and  \(\delta\) represents the ``strength'' of noise. When $\delta=0$, this condition  implies that the two decision rules are related through the  transformation $G_{f_Q^*}=h(G_{f_P^*}, \theta^*)$. Direct modeling of the decision boundary offers distinct advantages compared to traditional regression-function-based transfer learning approaches. Specifically, it allows flexible accommodation of geometric changes, such as translations and rotations, without imposing stringent form assumptions on the underlying regression function, thereby providing enhanced adaptability in real-world scenarios characterized by complex distributional drifts. As we'll explain in greater detail, when the dimension of the parameter $\theta^*$ is much smaller than the dimension of the features, our model essentially reduces a high-dimensional classification task to a lower-dimensional one by restricting the class of Bayes decision rules in (\ref{eq: transformequation}).

The noise component \(\delta\) serves a role analogous to that of divergence measure in the domain adaptation literature such as \cite{david2010impossibility}, quantifying similarity between the two sets $G_{f_Q^*}$ and $h(G_{f_P^*}, \theta^*)$. The noise component in our model explicitly acknowledges potential model misspecification and the inherent uncertainty in real-world transformations between decision boundaries. By doing so, our approach is robust against minor deviations or inaccuracies in specifying the transformation function, thus reflecting empirical settings where perfect alignment between domains rarely occurs.

Under the decision rule drift model (\ref{eq: transformequation}), we propose a novel methodology for constructing a classifier for the target distribution $Q$. The main idea is to calibrate the SVM-based classifier under the source distribution $P$ via an additional empirical risk minimization step. To avoid negative transfer, we also train a classifier on the target data only and then select among the classifiers based on target-data performance, following a similar idea as \cite{tsybakov2004optimal,bunea2007aggregation,reeve2021adaptive}. Under mild conditions (e.g., margin conditions and geometric noise conditions), the proposed approach achieves consistent classification, with error bounds explicitly dependent on the sample sizes $n_P$ and $n_Q$, the magnitude of the noise $\delta$, the dimensionality of $\theta^*$ and the parameters associated with the margin conditions and geometric noise conditions.

Motivated by its conceptual and practical parallels, our framework naturally extends to the estimation of optimal individualized treatment rule (ITR). Although randomized experiments are the gold standard for deriving ITR, their results may fail to generalize if the target population differs from the studied population \citep{zhao2019robustifying}. Recent literature addresses such limitations through domain adaptation and covariate drift approaches \citep{zhao2019robustifying,wu2023transfer}, but posterior drift remains largely unexplored. Our proposed decision-rule-based framework directly addresses this limitation, thereby significantly enhancing ITR estimation under posterior drift. We generalize our methodology and theory to the estimation of ITRs.

The remainder of this paper is structured as follows. Section~\ref{sec: method} introduces our proposed methodology for estimating the optimal classifier under the proposed decision rule transformation framework. Section~\ref{sec: theory} details the assumptions underpinning our approach and presents theoretical error bounds that characterize its performance. Section~\ref{sec: itr} further elaborates on the application of our algorithm to causal inference, specifically emphasizing its implications for ITR estimation and providing relevant theoretical guarantees. Finally, Sections~\ref{sec: simulation} and \ref{sec: realdata} illustrate the practical effectiveness of our methodology through comprehensive simulation studies and analyses of real-world data, respectively. All supporting code is available upon request and, upon publication, will be made publicly accessible on GitHub, with the methodology distributed as an \texttt{R} package.

\subsection{Notations}

Let \(\mathbb{I}(\cdot)\) be the indicator function, and \((\cdot)_+\) be the truncation function at zero such that \((x)_+ = \max(0, x)\). Let $a\wedge b=\min (a, b)$. The symmetric difference of two sets \(G_{f_1}, G_{f_2}\) is defined as:
\begin{align*}
d_\Delta(G_{f_1}, G_{f_2})= G_{f_1}\Delta G_{f_2} = (G_{f_1}\cap G_{f_2}^C) \cup \left( G_{f_1}^C\cap G_{f_2} \right).
\end{align*}
For a reproducing kernel Hilbert space (RKHS) \(\mathcal{H}\), the corresponding norm is denoted as \(\|\cdot\|_{\mathcal{H}}\). We use $\|v\|$ to denote the $L_2$ norm of a vector $v$.

In addition, we replace \(P\) by \(P_n\) when considering the empirical version of the risk. For example, the empirical version of the risk function is defined as:
\begin{align*}
 R_{P_n}(f) = \frac{1}{n} \sum_{i=1}^{n}\mathbb{I}\left(Y_i \neq \textrm{sign}( f(X_{i}))\right).
\end{align*}
Equivalently, for a decision rule \(G\), the empirical risk is defined as:
\begin{align*}
 R_{P_n}(G) = \frac{1}{n} \sum_{i=1}^{n}\mathbb{I}\left(Y_i \neq (2\mathbb{I}(X_i \in G) -1)\right).
\end{align*}
Similarly, we can define $R_P(f)$ and $R_P(G)$. The distance of a point \(x\in \mathcal{X}\) to the decision boundary is defined as:
\begin{align*}
    \tau(x) = \begin{cases}
 d(x, X_1) & \text{ if } x\in X_{-1},\\
 d(x, X_{-1}) & \text{ if } x\in X_{1},\\
        0 & \text{ otherwise},
    \end{cases}
\end{align*}
where \(d(x, X_1) \coloneqq \min_{y \in X_1} \left\| x-y \right\|\), and \(X_1 \coloneqq \left\{ x \in \mathcal{X}: f^*_P(x) > 0 \right\}\), and all other terms are defined similarly. The Hausdorff distance of two sets \(X, Y\) is defined as \(d_{H}(X,Y) := \max\left\{\,\sup_{x \in X} d(x,Y),\ \sup_{y \in Y} d(X,y) \,\right\}.\) We also consider the \(L_2\) norm of \(f\) with respect to the probability distribution \(P\) as: 
\begin{align*}
    \|f\|_{P, 2} \coloneqq \left( \int |f(X)|^2 dP_X\right)^{1/2}.
\end{align*}
Lastly, for a set \(A\), \(\bar{A}\) and $A^C$ denote the closure and the complement of \(A\).

\section{Proposed Methodology}\label{sec: method}
\subsection{Problem Setup}
In this section, we present the statistical setup and the proposed methodology. We observe two datasets \(D_P = \{(X_{i}^P, Y_i^P)\}_{i= 1}^{n_P} \sim P \) drawn i.i.d. from the source distribution \(P\)
and \(D_Q = \{(X_{i}^Q, Y_i^Q)\}_{i= 1}^{n_Q} \sim Q \) drawn i.i.d. from the target distribution \(Q\). For each observed  tuple \((X_i, Y_i)\), \(X_i\) is a \(d-\)dimensional vector of covariates and \(Y_i\) is the binary class indicator where \((X_i, Y_i)\in \mathcal{X}\times \left\{-1 , 1 \right\}, \ \mathcal{X} \subseteq \R^d\).

Recall that given a classifier $f$, the risk function under the target distribution $R_Q(f)$ and the Bayes risk $R^*_Q$ are defined in (\ref{eq_risk}) and (\ref{eq_risk_bayes}), respectively. A good classifier $\hat f_Q$ should have small excess risk, defined as:
\begin{align*}
 R_Q(G_{\hat{f}_Q}) - R^*_Q.
\end{align*}
Our goal is to construct a good classifier $\hat f_Q$ for the target distribution under the decision rule drift model \eqref{eq: transformequation}. Next, we'll introduce the methodology and corresponding theoretical results.

\subsection{Methodology}
To derive the decision rule, we focus on the SVM which directly estimates the decision function \citep{steinwart2007fast}. Unlike other nonparametric models for regression functions, SVM explicitly targets the estimation of the decision boundary function, thereby providing a straightforward representation of the decision rule. 

Specifically, we start from an initial classifier for the source data through SVM with the following Gaussian radial basis function (RBF), 
\begin{align*}
 k_\sigma(x, x') = \exp(-\sigma^2\|x - x'\|_2^2), \quad x, x'\in \mathcal{X},
\end{align*}
where \(\sigma > 0\) is the \(bandwidth\) parameter of the kernel. Let \(\mathcal{H}_\sigma\) be the RKHS of the Gaussian RBF kernel with bandwidth \(\sigma\) over \(\mathcal{X}\) of measurable functions. The corresponding classifier \(\hat{f}_P\) obtained through SVM is thus defined as:
\begin{align}
  \hat{f}_P &= \argmin_{f\in \mathcal{H_{\sigma_P}}} \lambda_P \left\|f\right\|_\mathcal{H_{\sigma_P}}^2 + \frac{1}{n_P}\sum_{i\in D_{P}}\Big\{ (1- Y_if(X_i) )_+\Big\},
\end{align}
where \(\lambda_P\) is a regularization parameter and \(\sigma_P\) is the bandwidth. We define $G_{\hat{f}_P}=\{x: \hat{f}_P(x)\geq 0\}$. We note that it is also possible for us to use the SVM classifier with an offset term, where a similar result will follow.

To transfer knowledge from the source data, we next estimate the unknown parameter $\theta^*$ in the decision rule drift model (\ref{eq: transformequation}). The main idea is to calibrate the set $G_{\hat{f}_P}$ from the above SVM classifier via an additional empirical risk minimization step on the target data. Before that,  we split the target dataset randomly into two equally sized subsets \(D_{1,Q}, D_{2,Q}\), reserving one subset to guard against negative transfer. In the dataset $D_{1,Q}$, we solve the following empirical risk minimization problem
\begin{align}\label{eq: ftildeqdef}
  \hat{\theta} &= \argmin_{\theta\in\Theta} \frac{2}{n_Q}\sum_{i\in D_{1, Q}} \mathbb{I}\left\{ Y_i \neq (2\mathbb{I}(X_i \in h(G_{\hat f_P},\theta)) -1) \right\} .
\end{align} 
Given the estimator $\hat{\theta}$, we define the calibrated decision rule as \(G_{\tilde{f}_Q} \coloneqq h(G_{\hat{f}_P}, \hat{\theta})\). The estimator $\hat{\theta}$ is obtained by minimizing the empirical misclassification error with respect to a low-dimensional parameter $\theta$ in a suitable parameter space $\Theta$. To better understand the empirical risk minimization problem (\ref{eq: ftildeqdef}), we consider the case that there exists some $\theta^*\in\Theta$ such that $G_{f_Q^*}=h(G_{f_P^*}, \theta^*)$, that is the decision rule drift model (\ref{eq: transformequation}) holds with $\delta=0$.  Since $G_{f_Q^*}$ corresponds to the Bayes classifier, we have $G_{f_Q^*}=\argmin_G Q\{Y\neq (2\mathbb{I}(X\in G) -1) \}$, which implies
\begin{align}\label{eq: ftildeqdef_2}
\theta^*=\argmin_{\theta\in\Theta} Q\left\{ Y\neq (2\mathbb{I}(X\in h(G_{f^*_P},\theta)) -1) \right\}. 
\end{align} 
The empirical risk minimization problem (\ref{eq: ftildeqdef}) can be viewed as the sample analog of (\ref{eq: ftildeqdef_2}). We note that, in general, we cannot replace the indicator function in (\ref{eq: ftildeqdef}) with a convex surrogate loss (i.e., hinge loss or logistic loss) \citep{bartlett2006convexity}. Otherwise, without further conditions, the resulting calibrated decision rule \(G_{\tilde{f}_Q}\) may not converge to $G_{f_Q^*}$ in terms of the excess risk. Since \(\theta\) is often of low dimension, the minimizer can be identified via the Nelder–Mead method or mixed-integer programming.

From the above intuition, we can see that the calibrated decision rule \(G_{\tilde{f}_Q}\) is a good classifier, when the noise component $\delta$ in the decision rule drift model (\ref{eq: transformequation}) is small. As expected, the excess risk of \(G_{\tilde{f}_Q}\) may inflate as $\delta$ grows, which eventually makes the resulting classifier inconsistent. This phenomenon is often referred to as negative transfer in the literature. To prevent negative transfer, we re-estimate the decision boundary using SVM with the target data only, 
\begin{align}
  \hat{f}_Q &= \argmin_{f\in \mathcal{H_{\sigma_Q}}} \lambda_Q \left\|f\right\|_\mathcal{H_{\sigma_Q}}^2 + \frac{2}{n_Q}\sum_{i\in D_{1, Q}} \Big\{ (1- Y_if(X_i) )_+\Big\},
\end{align}
where \(\lambda_Q\) is a regularization parameter and \(\sigma_Q\) is the bandwidth. Similarly, we define $G_{\hat{f}_Q}=\{x: \hat{f}_Q(x)\geq 0\}$. Following the idea in \cite{tsybakov2004optimal,bunea2007aggregation,reeve2021adaptive}, we aggregate the three decision rules, $G_{\tilde{f}_Q}, G_{\hat{f}_Q}$ and $G_{\hat{f}_P}$ by selecting the one with the smallest empirical risk in the dataset $D_{2,Q}$, 
\begin{align}
 G_{\hat{f}_{Q, final} }= \argmin_{G\in\{G_{\tilde{f}_Q}, G_{\hat{f}_Q}, G_{\hat{f}_P}\}} \frac{2}{n_Q}\sum_{i\in D_{2, Q}} \mathbb{I}\Big\{ Y_i \neq (2\mathbb{I}(X_i \in G) -1) \Big\} , \label{eq: fhatfinal}
\end{align}
which gives our final estimator.

\section{Theoretical Properties}\label{sec: theory}
We now establish theoretical properties of our classifier. We begin by listing the assumptions.
\subsection{Assumptions}\label{subsubsec: assu}
\begin{definition}[Margin Condition]\label{assu: margincondition}
 For distribution \(P: X\times Y\) with the regression function \(\eta_P\), we say that the margin condition holds with parameter \(\alpha\) if there exist constants \(\alpha\geq 0\) and \(C_\alpha>0\) such that
  \begin{align*}
      \pr(|2\eta_P(X)-1| \leq t) \leq C_\alpha t^\alpha \hspace*{.3in}  \forall 0 < t \leq t_0,
  \end{align*}
where \(t_0 \leq 1\) is a constant.
\end{definition}
This condition, which controls the mass around the decision boundary, has been discussed in \cite{mammen1999smooth} and \cite{tsybakov2004optimal}. When \(\alpha\) is large, the mass in regions where classification is inherently challenging is low. In extreme cases where \(\pr(|2\eta(X)-1| \geq \varepsilon) = 1\) for some \(\varepsilon\), the coefficient \(\alpha = \infty\). This assumption is widely applied in SVM and transfer learning for classification literature, see \cite{steinwart2007fast,hanneke2019value,cai2021transfer,reeve2021adaptive,fan2023robust,maity2024linear}. We denote the corresponding class of distributions as \(M(\alpha, C_\alpha)\). 
\begin{definition}[Modified Geometric Noise Assumption]\label{assu: noisecondition}
 For distribution \(P: X\times Y\) with the regression function \(\eta_P\), we say the modified geometric noise condition holds if there exist constants \(C_n > 0\) and \(\gamma > 0\) such that:
  \begin{align}\label{eq: margincondition}
      \int_{\left\{ x \in X: \tau(x)\leq t \right\}}|2\eta_P(x) -1| dP_X \leq C_n t^\gamma
  \end{align}
 holds for all \(0\leq t \leq t_0\) for \(t_0\in\R\) , where \(\tau(x)\) is the distance of point \(x\) to the Bayes decision boundary.
\end{definition}

The geometric noise condition is introduced in \cite{steinwart2007fast} to derive the fast rates of SVM. Here we present a slightly modified version from \cite{hamm2021adaptive} to simplify the analysis of the approximation error term. In addition, we note that this assumption is not implied by the margin condition, unless some additional assumption between the \(\tau(x)\) and \(|\eta(x) - 1/2|\) is present, such as the one presented in \cite{castro2007minimax}.x To better illustrate the difference between the two conditions, we rewrite the margin condition as 
\[ \int_{\left\{ x \in X: |\eta_P(x) - 1/2|\leq t \right\}} dP_X \leq C_\alpha t^\alpha.\] 
We can see that the margin condition keeps the density low, when the point is close to the decision boundary measured by the magnitude of \(|\eta(x) - 1/2|\). 
On the other hand, the modified geometric noise condition keeps the density low, when the point is close to the decision boundary measured by the Euclidean distance of \(x\) to the decision boundary. As an illustration, we provide a simple example in which we verify the  margin and modified geometric noise conditions. 
For simplicity, we denote the class of distributions satisfying the modified geometric noise condition as \(N(\gamma, C_\gamma)\).

\begin{example}\label{example: noninterchangeablity of geometric and exponential}

Consider a distribution \(P: X\times Y\) such that the covariate \(X\) is uniformly distributed within a \(d\)-dimensional unit ball centered at the origin, denoted by \(B_d(0, 1)\), and the regression function is given by
\[
\eta_P(x) = \frac{1}{2} + \frac{1}{2}\text{sign}(x^T\beta)|x^T\beta|^t,
\]
where \(t > 0\) is a fixed constant and \(\|\beta\| = 1\). In this example, the Bayes decision boundary is $\{x\in B_d(0, 1): x^T\beta=0\}$. The modified geometric noise condition is verified as follows: 
\begin{align*}
\int_{\left\{ x \in B_d(0, 1): \tau(x)\leq c \right\}}|2\eta_P(x) -1| dP_X  &=\int_{\left\{ x \in B_d(0, 1): \tau(x)\leq c \right\}}|x^T\beta|^t dx\\
&= \int_{ |x_1| \leq c} |x_1|^t d x\int_{\lVert x_{-1} \rVert \leq \sqrt{1 - x_1^2}} dx_{-1}\\
& = \frac{\pi^{(d-1)/2}}{\Gamma\left( \frac{d+1}{2} \right)}\int_{\lvert x_1 \rvert \leq c} \lvert x_1 \rvert^t (1-x_1)^{(d-1)/2}\, dx_1\\
  & \lesssim \int_{\lvert x_1 \rvert \leq c} \lvert x_1 \rvert^t\, dx_1\\
  & \lesssim c^{t+1}
\end{align*}
for some small constant $0<c<1$, where the second line follows from the fact that we rotate the axes to match $x^T\beta$ with the first coordinate being $x_1$ and denote $x_{-1}=(x_2,..,x_d)$.  Thus, the modified geometric noise condition is satisfied with \(\gamma=t + 1\). Using a similar calculation, we can show that the margin condition is satisfied with \(\alpha=1/t\). It is then easy to see when $t$ increases, the coefficient in the modified geometric noise condition increases, whereas the coefficient in the margin condition decreases, and thus the two conditions share different characteristics.
\end{example}

\begin{definition}[Strong Density Condition]
\label{def: density}
The distributions $P_X, Q_X$ are said to have common support and satisfy the strong density assumption with parameter $\mu=$ $\left(\mu_{-}, \mu_{+}\right), \mu_->0 \text{ and } \mu_+>0$ if both $P_X$ and \(Q_X\) are absolutely continuous with respect to the Lebesgue measure \(\lambda\), and
$$
\begin{aligned}
\Omega_X & \coloneqq \operatorname{supp}\left(P_X\right)=\operatorname{supp}\left(Q_X\right)\subset \R^d, \\
\mu_{-} & <\frac{d P_X}{d \lambda}(x)<\mu_{+} \quad \forall x \in \Omega_X,\\
\mu_{-} & <\frac{d Q_X}{d \lambda}(x)<\mu_{+} \quad \forall x \in \Omega_X.
\end{aligned}
$$
Lastly, we assume that \(\Omega_X\) is compact and has finite volume.
\end{definition}

The strong density assumption was first introduced for classification in \cite{audibert2007fast}, and also used in the transfer learning literature in \cite{cai2021transfer}. This assumption shows that the marginal distributions of \(P_X\) and \(Q_X\) share some similarities, since our paper focuses on the posterior drift setting. For the pair of distributions \((P, Q)\) satisfying the definition above, we  denote it as \(S(\mu_-, \mu_+)\).

\begin{definition}[Decision Rule Drift]\label{assu: boundaryassu}
Let \(f_P^*, f_Q^*\) be the Bayes classifiers of \(P\) and \(Q\) respectively. Let \(G_{f_P^*}\) and \(G_{f_Q^*}\) denote the corresponding decision sets. We assume that there exists a transfer function \(h(\cdot,\cdot)\) indexed by some parameter \(\theta^* \in \Theta \subseteq \R^p\) such that:
\begin{align*}
 d_{H}(G_{f_Q^*}, h(G_{f_P^*}, \theta^*))\vee d_{H}\left( \overline{G_{f_Q^*}^C}, \overline{h(G_{f_P^*}, \theta^*)^C} \right)\leq \delta,
\end{align*}
where \(d_{H}\) is the Hausdorff distance, and \(h\) is Lipschitz with constants \(M_1, M_2\) under the symmetric difference measure. That is, for any set $G, G_1, G_2$, and parameter $\theta, \theta_1$ and $\theta_2$, we have
\begin{align*}
 d_\Delta(h(G, \theta_1), h(G, \theta_2)) &\leq M_1 \left\| \theta_1- \theta_2 \right\|_2,\\
 d_\Delta(h(G_1, \theta), h(G_2, \theta)) &\leq M_2 d_\Delta(G_1,  G_2).
\end{align*}
Without loss of generality, we assume $h(G,0)=G$ for any set $G$. 
\end{definition}
The above definition formalizes our decision rule drift model (\ref{eq: transformequation}). The Lipschitz property under the symmetric difference measure indexed by a parameter holds for several common set operations, such as translation, rotation, or a composite of both \citep{schymura2014upper}. We denote the class distributions that follow the above assumption as \(\mathcal{H}(\Theta, \delta)\).

\begin{remark}
Our proposed decision rule drift model shares a similar spirit of the posterior drift model in \cite{maity2024linear}. However, the posterior drift model in \cite{maity2024linear} imposed a parametric transformation between the two regression functions $\eta_P$ and $\eta_Q$, whereas our proposed decision rule drift model can be viewed as a more direct approach to modeling the relevance of the two decision sets $G_{f^*_P}$ and $G_{f^*_Q}$. In addition, by allowing for $\delta>0$, our method can accommodate misspecification of the transfer function  \(h(\cdot,\cdot)\), that is, $G_{f_Q^*}\neq h(G_{f_P^*}, \theta^*)$, which shares similar characteristics as domain adaptation literature. 
\end{remark}
\begin{remark}
In our approach, we adopt the Hausdorff distance to quantify the similarity between the two decision sets $G_{f^*_P}$ and $G_{f^*_Q}$ when $h(\cdot,\theta^*)$ is an identity function. In the domain adaptation literature, \cite{ben2006analysis} measured the similarity between the source and target distributions via the \(\mathcal{A}-\)distance \(d_\mathcal{A}(P, Q) \coloneqq 2 \sup_{A \in\mathcal{A}}|P(A) - Q(A)|\), where \(\mathcal{A}\) is a suitable collection of measurable sets. However, they did not consider the existence of transformation in the decision rule and thus the method doesn't directly apply to our model. In addition, as we'll see in the discussion of the main result, our formulation provides a faster rate of convergence under proper additional assumptions.
\end{remark}

Putting the above assumptions together, we consider the following model space:
\begin{align*}
 & \Pi(\alpha_P, C_{\alpha_P},\alpha_Q, C_{\alpha_Q},\gamma_P, C_{\gamma_P}, \gamma_Q, C_{\gamma_Q}, \mu_-, \mu_+, \Theta, \delta)  \\
 & := \left\{ (P, Q): P \in M(\alpha_P, C_{\alpha_P})\cap N(\gamma_P, C_{\gamma_P}), Q\in M(\alpha_Q, C_{\alpha_Q})\cap N(\gamma_Q, C_{\gamma_Q}),\right.\\
 &~~~~~\phantom{A}\left. (P, Q)\in S(\mu_-, \mu_+)\cap\mathcal{H}(\Theta, \delta)\right\}.
\end{align*}
For simplicity, we denote  \(\Pi(\alpha_P, C_{\alpha_P},\alpha_Q, C_{\alpha_Q},\gamma_P, C_{\gamma_P}, \gamma_Q, C_{\gamma_Q}, \mu_-, \mu_+, \delta) \) as \(\Pi\) if no confusion arises. Different from the regression function based approach in \cite{cai2021transfer},  we introduce two sets of parameters, in the margin condition, $(\alpha_P, C_{\alpha_P})$ and $(\alpha_Q, C_{\alpha_Q})$ for the source and target distributions, respectively, as the margin condition on \(P\) cannot be transferred to \(Q\) or vice versa under our model. This is because our decision rule drift model (\ref{eq: transformequation}) and the margin condition characterize different aspects of the model formulation. Specifically, our decision rule drift model only imposes the constraint on the geometric structure of the decision boundary, whereas the margin condition pertains to the behavior of the regression function nearby the decision boundary. Following similar arguments, we also impose two sets of parameters in the modified geometric noise condition for the source and target distributions, respectively.

\subsection{Main Theoretical Results}
In this subsection, we present our main result on the convergence rate of the transfer learning classifier proposed in Section~\ref{sec: method}.
\begin{theorem}\label{thm: fhatfinalresult}
 Let \(G_{\hat{f}_{Q, final}}\) be the classifier defined by \eqref{eq: fhatfinal}. Denote \(\gamma_P' = \gamma_P/d\), \(\gamma_Q' = \gamma_Q/d\), 
  \begin{align*}
    \beta_P := \begin{cases}
      \frac{\gamma_P'}{2\gamma_P'+1} & \text{ if } \gamma_P' \leq \frac{\alpha_P  + 2 }{2 \alpha_P}\\
      \frac{2\gamma_P'(\alpha_P+1)}{2\gamma_P'(\alpha_P+2) + 3 \alpha_P + 4}&\text{ otherwise. }
    \end{cases}
  \end{align*}
Let \(\lambda_{n_P, P}= n_P^{-(\gamma_P'+d)/\gamma_P'\beta_P}\) and \(\sigma_{n_P, P}= n_P^{\beta_P/\gamma_P'}\).   Define \(\beta_Q,\lambda_{n_Q, Q}\) and $\sigma_{n_Q, Q}$ in the same way. In addition, let 
\begin{align*}
 A_{n_P, n_Q} &=  \left(n_P^{-\beta_P+\varepsilon_P} \right)^{{\frac{\alpha_P}{1+\alpha_P}}} +   \delta^{\gamma_Q} +  \left( \frac{p\log n_Q}{n_Q} \right)^{\frac{1+ \alpha_Q}{2+ \alpha_Q}} \wedge \left\| \theta^* \right\|_2 ,\\
 B_{n_Q} &= n_Q^{-\beta_Q+\varepsilon_Q},
  \end{align*}
for any \(\varepsilon_P, \varepsilon_Q > 0\). Then, uniformly over $(P,Q)\in \Pi$, 
  \begin{align}
{R}_{Q}(G_{\hat{f}_{Q, final}})  -  {R}_{Q}^*
    =O_p\left( A_{n_P, n_Q}\wedge B_{n_Q}+ \left( \frac{1}{n_Q} \right)^{\frac{1+\alpha_Q}{2+\alpha_Q}}\right).\label{thm_bound}
  \end{align}
\end{theorem}
The proof of this theorem can be found in the Supplementary Materials. At a high level, we first establish the error rate of $\hat f_P$ under the modified geometric noise condition. Then, we employ empirical process techniques to derive performance guarantees for the ERM classifier. Finally, we aggregate the classifiers to obtain the final error bound. 

The upper bound for the excess risk can be interpreted as follows.  The component \(A_{n_P,n_Q}\), excluding \(\lVert \theta^* \rVert_2\), coincides with the risk bound for the calibrated classifier \(G_{\tilde f_Q}\): specifically, \(\bigl(n_P^{-\beta_P+\varepsilon_P}\bigr)^{\frac{\alpha_P}{1+\alpha_P}}\) is the convergence rate of the source SVM \(\hat f_P\); \(\delta^{\gamma_Q}\) quantifies the discrepancy between \(R_Q(G_{f_Q^*})\) and \(R_Q\!\bigl(h(G_{f_P^*},\theta^*)\bigr)\); and \(\bigl(\tfrac{p\log n_Q}{n_Q}\bigr)^{\frac{1+\alpha_Q}{2+\alpha_Q}}\) is the ERM estimation error in \eqref{eq: ftildeqdef}, with \(p=\dim(\theta^*)\). Setting \(\theta=0\) (so \(G_{\hat f_P}=h(G_{\hat f_P},0)\)) replaces this ERM term by \(\lVert \theta^* \rVert_2\) and recovers the risk bound for \(\hat f_P\). The error \(B_{n_Q}\) is the standard risk bound for the SVM classifier $\hat f_Q$ with Gaussian kernels, first introduced by \cite{steinwart2007fast}. Introducing such a term prevents possible negative transfer when our calibrated classifier $G_{\tilde f_Q}$ has a compromised performance. In addition, we note that the last term \( n_Q^{-\frac{1+\alpha_Q}{2+\alpha_Q}} \) in the risk bound for $G_{\hat{f}_{Q, final}}$ is the cost of the aggregation of the classifiers.

When \(n_P \gg n_Q\), $\delta$ is sufficiently close to 0 and $p$, the dimension of $\theta^*$, is small, the final rate of convergence is dominated by the error of the calibrated classifier $G_{\tilde f_Q}$, \(( \frac{p\log n_Q}{n_Q} )^{\frac{1+ \alpha_Q}{2+ \alpha_Q}}\) obtained through empirical risk minimization \citep{mammen1999smooth,tsybakov2004optimal}. In particular, this rate can be faster than $B_{n_Q}(x)$, the risk bound for the standard SVM classifier applied to the target data, illustrating the theoretical benefit of borrowing information from the source data to improve classification performance on the target data.  
This reinforces our insight, where our idea is to convert a nonparametric classification problem, which usually suffers from the curse of dimensionality, to a low-dimensional empirical risk minimization (assuming $p$ is small) through additional structure assumptions defined in the decision rule drift model.

Similar to the domain adaptation literature such as \cite{ben2006analysis,david2010impossibility}, the noise term \(\delta^{\gamma_Q}\) persists in the final rate. As mentioned previously, the noise term characterizes the effect of the model misspecification of the transfer function  \(h(\cdot,\cdot)\). Indeed, our noise term vanishes to 0 in a fast rate as $\delta\rightarrow 0$, if \(\gamma_Q\) is large enough.

\section{Individualized Treatment Rule Estimation}\label{sec: itr}

In this section, we explore the application of our proposed methodology to the estimation of optimal individualized treatment rules (ITRs). 
\subsection{Statistical Setup and Methodology}

Let \(P\) and \(Q\) denote the source and target distributions, respectively.  From \(P\) we draw an i.i.d. observations of size \(n_{P}\), and from \(Q\) we draw an i.i.d. observations of size \(n_{Q}\). Each observation $i$ is represented by a triplet $(X_i, T_i, R_i)$, where $X_i \in \mathbb{R}^d$ is the pre-treatment covariate vector, $T_i \in \{1,-1\}$ is the treatment indicator (1 if treated and -1 otherwise), and $R_i$ is the observed clinical outcome, where a larger value is desirable. Following the potential outcomes framework and under the stable unit treatment value assumption (SUTVA; \cite{rubin1990comment}), we define the potential outcomes as $R_i(1)$ and $R_i(-1)$ corresponding to the treatment and control conditions, respectively, and only \(R_i(T_i)\) is observed.

The optimal individualized treatment rule (ITR) for the target population $Q$ is defined by:
\begin{align}\label{eq_G_Q_ITR}
 G_{f_Q^*} = \argmax_{G\in\mathcal{G}}\mathbb{E}_{Q}\left[R(1)\mathbb{I}(X \in G) + R(-1)\mathbb{I}(X \notin G)\right],
\end{align}
where $\mathcal{G}$ is a given class of decision rules. Under the assumptions of unconfoundedness (Assumption~\ref{assu: unconf}) and overlap (Assumption~\ref{assu: overlap}) \citep{rosenbaum1983central}, we denote the propensity score in population $Q$ as $\pi_Q(X) = Q(T=1 \mid X)$. The expected clinical outcome can then be decomposed as:
\begin{align*}
  &\phantom{=}\E_{Q}\left[R(1)\mathbb{I}(X \in G) + R(-1)\mathbb{I}(X \notin G)\right] \\
  & = \E_{Q}\left[ R(-1) \right] +\E_Q\left[ \left(  \frac{R T}{T\pi_Q(X) + (1-T)/2}  \right)\mathbb{I}(X \in G)\right].
\end{align*}  
The maximizer can be rewritten as:
\begin{align*}
  G_{f_Q^*} & = \argmax_{G \in \mathcal{G}}\E_Q\left[\left(\frac{R T}{T\pi_Q(X) + (1-T)/2}\right)\mathbb{I}(X \in G)\right] \\
  & = \argmin_{G \in \mathcal{G}}\E_Q\left[\left(\frac{R}{T\pi_Q(X) + (1-T)/2}\right)\mathbb{I}(T\neq (2\mathbb{I}(X \in G)-1))\right].
\end{align*} 
 
With the above derivation, it is now clear that the estimation of optimal ITR can be considered as a weighted classification problem, as also illustrated in \cite{zhao2012estimating,kitagawa2018should}. To model the optimal ITR under the two distributions $P$ and $Q$, we consider the same decision rule drift model (\ref{eq: transformequation}), where $G_{f_Q^*}$ is defined in (\ref{eq_G_Q_ITR}) and $G_{f_P^*}$ is defined in a similar manner.

As a result, the classifier that maximizes the expected clinical outcome also minimizes the risk for the weighted classification problem, where the risk function is defined as:
\begin{align*}
  {R}^W_Q(G) = \E_Q\left[\left(\frac{R}{T\pi_Q(X) + (1-T)/2}\right)\mathbb{I}(T\neq (2\mathbb{I}(X \in G)-1))\right].
\end{align*}

Under the assumption that the propensity score models \(\pi_Q(X)\) and $\pi_P(X)$ are both known, we extend our proposed framework in Section \ref{sec: method} to estimate the optimal ITR, as formalized in Algorithm~\ref{alg:itr_transfer}. The algorithm presented follows a similar structure given in Section~\ref{sec: method}, except that we instead consider a weighted hinge loss in SVM and a weighted 0-1 loss in the empirical risk minimization. Under the similar assumptions as in Section 3, the weighted-classification excess risk of the estimated ITR admits the identical bound as in Theorem \ref{thm: fhatfinalresult}.

\begin{algorithm}[!ht]
  \caption{Transfer Learning for Individualized Treatment Rule under decision rule drift}
  \label{alg:itr_transfer}
  \begin{algorithmic}[1]
  
  \State \textbf{Input:} Source data $D_P = \{(X_i, T_i, R_i)\}_{i=1}^{n_P}$, Target data $D_Q = \{(X_i, T_i, R_i)\}_{i=1}^{n_Q}$.
  \State \textbf{Output:} Final individual treatment rule $G_{\hat{f}_{Q, final}}$

  \State Obtain the source domain classifier $\hat{f}_P$ using weighted SVM:
  \begin{align*}
  \hat{f}_P = \argmin_{f\in \mathcal{H}_{\sigma_P}} \frac{1}{n_P}\sum_{i\in D_P}  \left( \frac{R_i}{T_i\pi_P(X_i) + (1-T_i)/2}\right)(1- T_i f(X_i))_+ + \lambda_P \|f \|_{\mathcal{H}_{\sigma_P}}^2 .
  \end{align*}
  \State Define the corresponding decision set $G_{\hat{f}_P} = \{x : \hat{f}_P(x) > 0\}$.

  \State Split the target dataset $D_Q$ into two equal-sized subsets: $D_{1,Q}$ and $D_{2,Q}$.

  \State Obtain the initial target domain classifier $\hat{f}_Q$ using weighted SVM on $D_{1,Q}$:
  \begin{align*}
  \hat{f}_Q = \argmin_{f\in \mathcal{H}_{\sigma_Q}} \frac{2}{n_Q}\sum_{i\in D_{1, Q}}\left( \frac{R_i}{T_i\pi_Q(X_i) + (1-T_i)/2}\right)(1- T_i f(X_i))_+ + \lambda_Q \|f \|_{\mathcal{H}_{\sigma_Q}}^2 .
  \end{align*}
  \State Define the decision set $G_{\hat{f}_Q} = \{x : \hat{f}_Q(x) > 0\}$.

  \State Estimate the transformation parameter $\theta$ using $D_{1,Q}$:
  \begin{align*}
  \hat{\theta} = \argmin_{\theta\in\Theta} \frac{2}{n_Q}\sum_{i\in D_{1, Q}}  \frac{R_i}{T_i\pi_Q(X_i) + (1-T_i)/2}\mathbb{I}\left( T_i \neq  (2\mathbb{I}\{ X_i \in h(G_{\hat{f}_P}, \theta)\}-1) \right).
  \end{align*}
  \State Define the corresponding decision set $G_{\tilde{f}_Q} = h(G_{\hat{f}_P}, \hat{\theta})$.

  \State Aggregate the classifiers $G_{\tilde{f}_Q}$, $G_{\hat{f}_Q}$, and $G_{\hat{f}_P}$ using $D_{2,Q}$:
  \begin{align}\label{eq: itr_final}
 G_{\hat{f}_{Q, final}} = \argmin_{G\in\{G_{\tilde{f}_Q}, G_{\hat{f}_Q}, G_{\hat{f}_P}\}} \frac{2}{n_Q}\sum_{i\in D_{2, Q}}  \left( \frac{R_i}{T_i\pi_Q(X_i) + (1-T_i)/2}\right)\mathbb{I}\left( T_i \neq (2\mathbb{I}\{ X_i \in G \} -1) \right).
  \end{align}
  
  \State \textbf{return} $G_{\hat{f}_{Q, final}}$
  \end{algorithmic}
  \end{algorithm}

\subsection{Theoretical Results}

Now, we introduce the theoretical results of our methodology. In addition to SUTVA, we start with some standard assumptions based on the potential outcome framework \citep{rosenbaum1983central}.

\begin{assumption}[Unconfoundedness] \label{assu: unconf}
 The treatment assignment is unconfounded, i.e.,\phantom{A} \(\{R_{i}(-1), R_{i}(1)\} \indep T_{i} \ | \ X_{i}\).
\end{assumption}
\begin{assumption}[Strict Overlap]\label{assu: overlap}
 There exists a constant \(c_0 > 0\) such that \(c_0 \leq P(T_{i}= 1| X_{i})\leq 1-c_0\), and \(c_0 \leq\ Q(T_{i}= 1| X_{i}) \leq 1-c_0\).
\end{assumption}

Assumption~\ref{assu: unconf} requires that there is no unmeasured confounder, which is usually satisfied for experiments with random assignments. Assumption~\ref{assu: overlap} implies that every sample has a positive probability to receive the treatment or belong to the control group. When Assumptions~\ref{assu: unconf} and \ref{assu: overlap} are satisfied, the treatment assignment is considered as strongly ignorable \citep{rosenbaum1983central}. The above two assumptions are standard in the causal inference literature.

\begin{assumption}[Bounded Outcome] \label{assu: boundedoutcome}
 There exists a finite constant $M$ such that the support of outcome variable \(R\) is contained in \([-M, M]\).
\end{assumption}

This assumption is mainly used for derivation of our theoretical results which is also seen in \cite{kitagawa2018should}. To adapt the assumptions in Section~\ref{sec: theory}, we define 
\begin{align*}
  \phi(X) \coloneqq \frac{\E(R(1)|X)-\E(R(-1)|X)}{\E(R(1)|X)+\E(R(-1)|X)},
\end{align*}
which serves an analogous role to the regression function in classification problems. This definition generalizes to both source (\(P\)) and target (\(Q\)) populations, with domain-specific variants denoted as \(\phi_P\) and \(\phi_Q\), respectively.

\begin{assumption}[Margin Condition]\label{assu: itr_margincondition}
 There exist constants \(\alpha_P, \alpha_Q, C_\alpha, t_0\) such that 
  \begin{align*}
 P(|\phi_P(X)| \leq t) \leq C_\alpha t^{\alpha_P},~~ Q(|\phi_Q(X)| \leq t) \leq C_\alpha t^{\alpha_Q} 
  \end{align*}
 for all \(0 \leq t \leq t_0, t_0\leq 1\).
\end{assumption}
\begin{assumption}[Modified Noise Condition]\label{assu: itr_modifiednoise}
 There exist constants \(\gamma_P, \gamma_Q\) and constant \(C_n\) such that:
  \begin{align*}
    \int_{x\in X: \tau_P(x)\leq t} |\phi_P(x)| dP_X &\leq C_n t^{\gamma_P},~~ \int_{x\in X: \tau_Q(x)\leq t} |\phi_Q(x)| dQ_X \leq C_n t^{\gamma_Q}
  \end{align*}
 hold for all \(0\leq t \leq t_0\) for \(t_0\in\R\) , where \(\tau_P(x)\) is the distance of point \(x\) to the decision boundary $\{x: \phi_P(x)=0\}$ and \(\tau_Q(x)\) is defined similarly.
\end{assumption}

\begin{assumption}
  \label{assu: itr_boundaryassu}
We assume that the optimal ITR, $G_{f_Q^*}$ defined in (\ref{eq_G_Q_ITR}) and $G_{f_P^*}$ satisfy the decision rule drift model in Definition \ref{assu: boundaryassu}. 
\end{assumption}

We refer the readers to Section~\ref{sec: theory} for the discussion of the above  assumptions. With the assumptions presented above, we are now ready to present the theoretical guarantee for the estimation of the optimal ITR rule.

\begin{corollary}
  Let \(G_{\hat{f}_{Q, final}}\) be the classifier defined by \eqref{eq: itr_final}. Under Assumptions 1-6, the same upper bound as in (\ref{thm_bound}) holds for the excess risk ${R}_{Q}^W(G_{\hat{f}_{Q, final}})  -  {R}_{Q}^W(G_{f_Q^*})$. 
\end{corollary}

The proof of this result can be done in an analogous manner as Theorem~\ref{thm: fhatfinalresult}. 

\section{Simulation}\label{sec: simulation}

In this section, we evaluate the performance of the proposed methodology through comprehensive simulations. For each simulation scenario, we generate 320 independent datasets and compare our approach with existing methods and simple pooling strategies. We examine various transformations of the decision boundary and explore different forms of the regression function. Furthermore, we investigate how changes in the feature dimension, sample size, and the magnitude of the boundary shift influence the performance of each method. Details of each simulation setting are specified below. The support vector machine (SVM) classifiers are obtained using the \texttt{e1071} package, and empirical risk minimization classifiers are computed via the Nelder–Mead method. For benchmarking, we implement the methods proposed in \cite{fan2023robust} and \cite{maity2024linear}. The \texttt{R} code used to reproduce these results is provided upon request and will be publicly available on GitHub upon publication. 

Both source and target data share the same marginal distribution \(P_X = Q_X\), uniformly distributed on \([-3,3]^d\), where we vary the dimension \(d \in \{3, 5, 8, 10, 15, 20\}\). When evaluating the effect of shift magnitude, we vary \(\theta = \left\{ -1/2, 1/2, 1, 2,3,4 \right\}\) for translation and \(\theta = \left\{ -\pi/12, \pi/12, \pi/6, \pi/3, \pi/2, 2\pi/3 \right\}\) for rotation. Lastly, when evaluating the impact of different data sizes, we set \(n_P = 2000\) and \(n_Q = 2000/\text{share}\), where the share ratio takes values in $\{2, 5, 8, 16, 32, 64\}$.
For all other cases, we fix \(n_P = 2000\), \(n_Q = 400\), and \(d = 5\).

We consider three configurations of the decision boundary:
\begin{enumerate}[label=\alph*)]
    \item Linear translation:
    \[
 G_P^* = \{\beta^\top x > 0\}, \quad G_Q^* = \{\theta + \beta^\top x > 0\},
    \]
 where \(\beta = (3, 1, \dots, 1)^\top\), and \(\theta\) controls the translation magnitude.
    \item Linear translation with noise:
  \[
 G_P^* = \{\beta^\top x > 0\}, \quad G_Q^* = \{\theta + \beta^\top x > \epsilon\},
    \]
 where \(\beta = (3, 1, \dots, 1)^\top\), $\epsilon|X=x\sim N(0,\sigma^2) I(x>0)$, and \(\theta\) controls the translation magnitude.

    \item Nonlinear rotation:
    \[
 G_P^* = \{x^\top Q x > 0\}, \quad G_Q^* = \{x^\top P^\top Q P x > 0\},
    \]
 where the matrices \(Q\) and \(P\) are defined as
    \[
 Q = \begin{pmatrix}
        0.3 & 0 & 0 & & \\
        0 & 0 & \frac{1}{2} & & \\
        0 & \frac{1}{2} & 0 & & \\
        & & & \ddots & \\
        & & & & 0
    \end{pmatrix}, \quad
 P = \begin{pmatrix}
        \cos(\theta) & -\sin(\theta) & & & \\
        \sin(\theta) & \cos(\theta) & & & \\
        & & 1 & & \\
        & & & \ddots & \\
        & & & & 1
    \end{pmatrix},
    \]
 and \(\theta\) represents the rotation magnitude.
\end{enumerate}
We systematically vary \(\theta\) to assess the robustness of the method across shifts. Regarding regression functions, we employ a logistic regression function for scenario (a) and a deterministic regression function for scenario (b) and (c), i.e., \(y_i = 1\) if \(x_i \in G^*\) (for the source data \(G^* = G_P^*\); for the target data $G^* = G_Q^*$) and $-1$ otherwise.

We compare the following methods:
\begin{itemize}
    \item \textbf{Proposed}: Our proposed approach detailed in Section~\ref{sec: method}.
    \item \textbf{Fan}: The transfer-around-boundary model from \cite{fan2023robust}, combining estimators by thresholding the estimated regression functions \(\eta_P, \eta_Q\).
    \item \textbf{Maity}: The linear adjustment method from \cite{maity2024linear}, initializing with logistic regression on the source data, and subsequently updating interaction terms using the target data.
    \item \textbf{Pooled}: The SVM trained jointly on combined source and target data.
    \item \textbf{Source Only}: The SVM trained exclusively on the source data.
    \item \textbf{Target Only}: The SVM trained exclusively on the target data.
\end{itemize}

The simulation results are summarized in Figures~\ref{fig:setting_a} for setting (a); Figures~\ref{fig:setting_b} for setting (b); Figures~\ref{fig:setting_c} for setting (c). {The misclassification rates are computed on a validation set independently generated from the target distribution with the same sample size as the target data.}  Additional simulation scenarios with alternative decision boundaries and regression functions are detailed in the supplementary materials.
\begin{figure}
  \centering
  \includegraphics{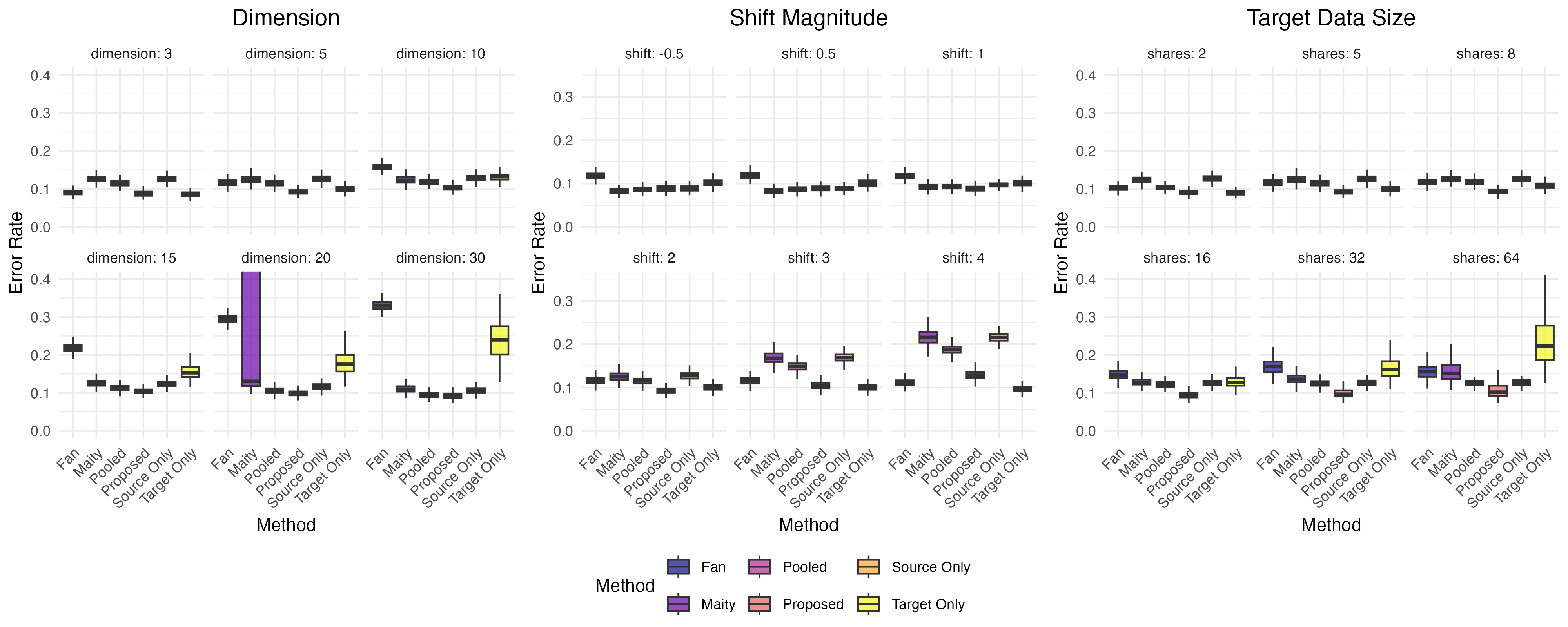}
  \caption{Misclassification rates  under different dimension, shift magnitude, and data sizes under setting (a). Six methods as specified in this section are compared.}
  \label{fig:setting_a}
\end{figure}
\begin{figure}
  \centering
  \includegraphics{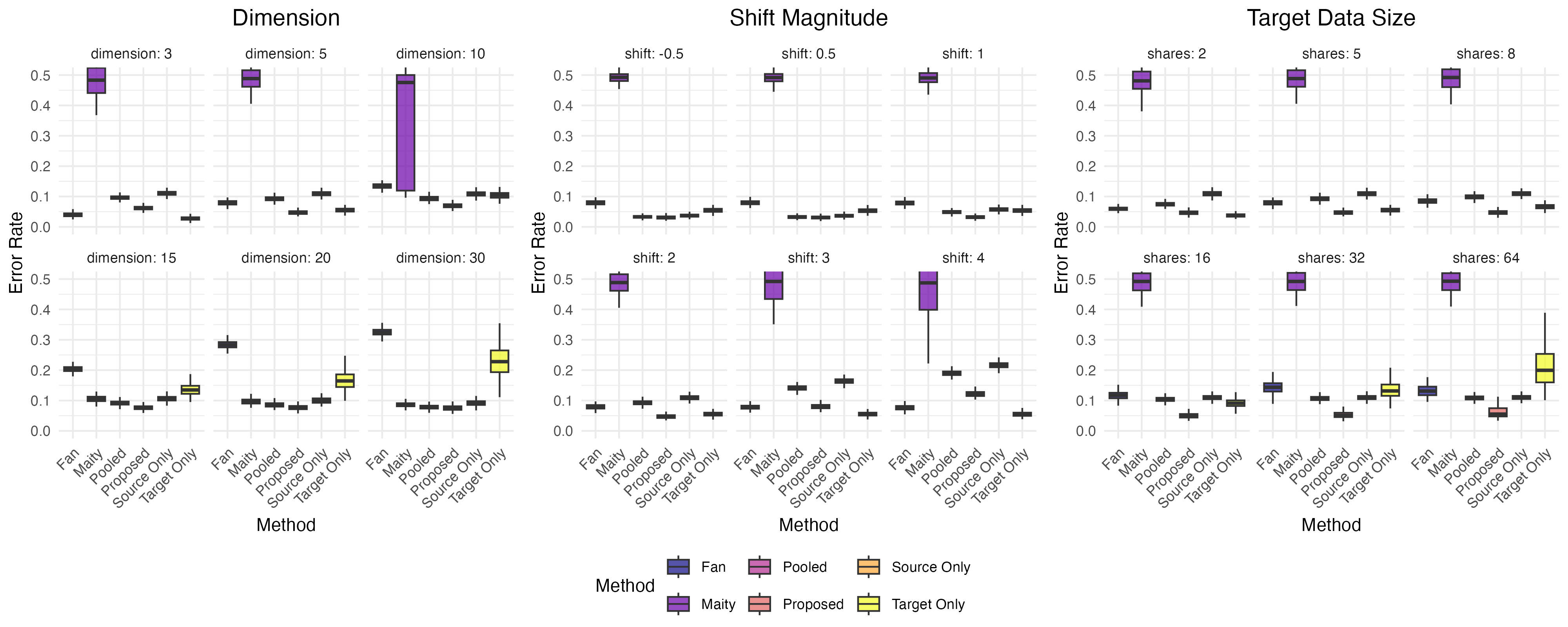}
  \caption{Misclassification rates  under different dimension, shift magnitude, and data sizes under setting (b). Six methods as specified in this section are compared.}
  \label{fig:setting_b}
\end{figure}
\begin{figure}
  \centering
  \includegraphics{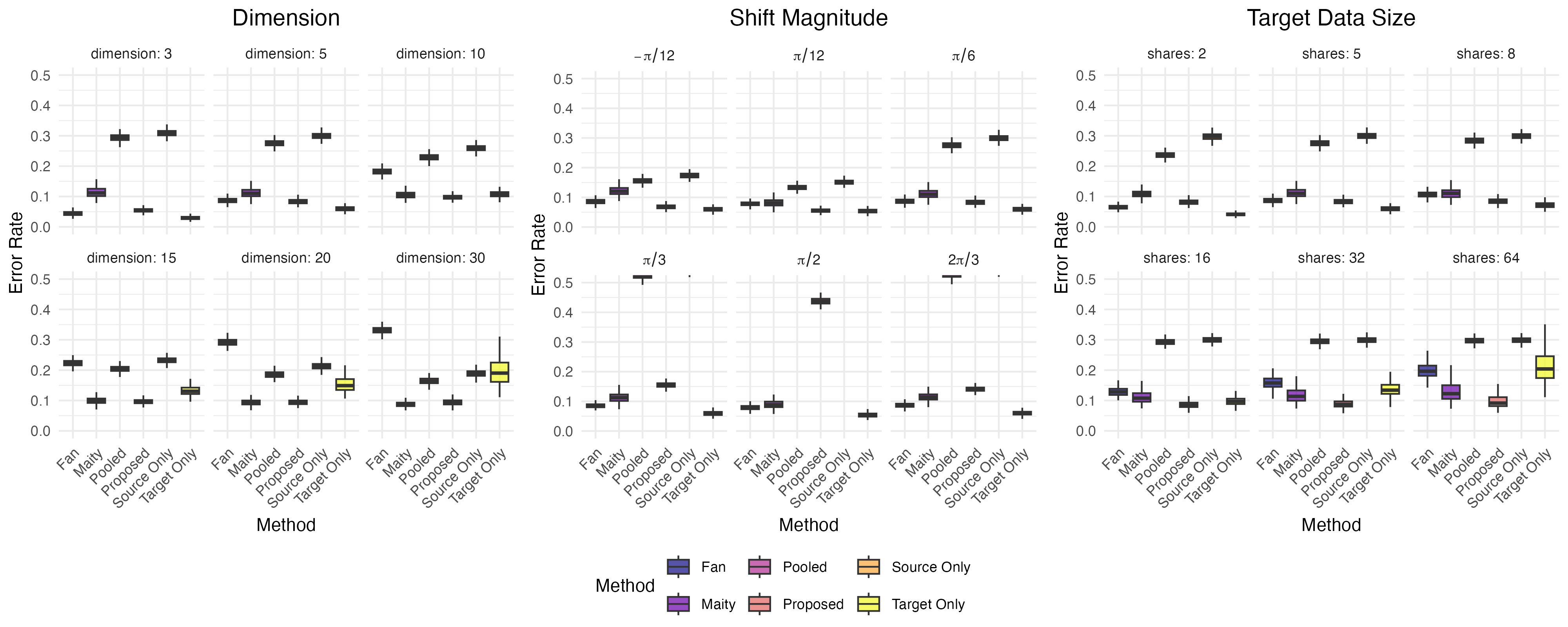}
  \caption{Misclassification rates under different dimension, shift magnitude, and data sizes under setting (c). Six methods as specified in this section are compared.}
  \label{fig:setting_c}
\end{figure}

As we can see from the plots, under most settings, the proposed method outperformed the existing methods, especially when \(n_P \gg n_Q\) and when the underlying decision boundary is nonlinear or the dimension is high. The reason is that, under settings (a) and (c), the proposed decision rule drift model holds with $\delta=0$, and thus our final classifier reduces to the calibrated classifier, which achieves faster rate than the competing methods. As for setting (b), the decision boundary noise is relatively small,  We also note that when the shift magnitude is large, the data can be extremely unbalanced, leading to poorer accuracy of the proposed classifiers.  In addition, as seen in Figure~\ref{fig:setting_a} and Figure~\ref{fig:setting_c}, the estimator in \cite{maity2024linear} can be unstable when the dimension of the covariates is high. This is because the method considers all possible interaction terms, and thus the Hessian matrix may become singular under higher dimension settings. Under setting (b) (Figures~\ref{fig:setting_b}), model misspecification may potentially lead to negative transfer, explaining the drop in performance. Due to space constraints, we only present a representative subset of the simulation results here; the full suite of experiments and additional numerical evidence are provided in the supplementary materials, which substantiate the patterns described above.

\section{Real Data Application}\label{sec: realdata}

We evaluate the performance of our proposed transfer-learning procedure on two observational datasets: the Current Population Survey (CPS) and the National Supported Work (NSW) demonstration program \citep{lalonde1986evaluating}, both of which are publicly available, and the code for analysis will be publicly available upon publication.

The CPS dataset comprises $n_S = 15,992$ individuals, whereas the NSW dataset contains $n_T = 445$ participants. In each dataset we observe baseline covariates including race (Black, White, Other), age (in years), marital status (Married vs.\ Not Married), and years of schooling, as well as realized earnings in 1974 ($\mathrm{RE}_{1974}$) and 1978 ($\mathrm{RE}_{1978}$). We define the binary treatment indicator
\[
A_i =
\begin{cases}
1, & \text{if individual }i\text{ has a high-school diploma},\\
0, & \text{otherwise},
\end{cases}
\]
and we set the outcome to
\[
Y_i = \log\bigl(1 + \mathrm{RE}_{i,1978}\bigr),
\]
to stabilize the right-skewed income distribution. We apply our transfer learning ITR estimator to two observational cohorts to assess whether borrowing information from CPS improves treatment recommendations for NSW participants.

In the NSW program, a subset of participants also received a job-training intervention, which is not recorded in the CPS. To ensure comparability of covariate distributions across the two samples, we restrict the NSW target sample to the $n_T' = 180$ individuals who did not receive job training. The resulting imbalance in sample sizes motivates our transfer learning approach, which borrows strength from the large CPS source data when estimating an individualized treatment rule for the smaller NSW cohort. We compare the following three decision rules:
\begin{enumerate}
  \item \textbf{Proposed}: transfer learning rule trained on CPS and adapted to the NSW sample;
  \item \textbf{Source-Only}: rule trained exclusively on CPS;
  \item \textbf{Target-Only}: rule trained exclusively on the restricted NSW sample.
\end{enumerate}

The propensity score is estimated through logistic regression on both populations, and the linear translation as in simulation setting (a) is considered. Table~\ref{tab:mean-log-income} reports the estimated value function for the log-transformed 1978 earnings in the NSW target sample under each policy. The proposed transfer-learning rule achieves a substantially higher average log-earnings (1.720) compared to both the source-only (0.305) and target-only (-0.580) rules. This gap indicates that directly applying a rule trained on the large CPS dataset to the NSW cohort yields modest gains, but neglects target-specific heterogeneity. Conversely, a rule trained solely on the small NSW sample overfits the limited data, resulting in negative average gains. By economically combining information from both datasets, our method captures general treatment effects from CPS while adapting to the unique covariate–outcome relationships in the NSW sample, thereby delivering superior treatment recommendations.

\begin{table}[ht]
  \centering
  \caption{Mean Log-Transformed 1978 Earnings (the bigger, the better)}
  \label{tab:mean-log-income}
  \begin{tabular}{lccc}
    \toprule
    Decision Rule     & Proposed & Source-Only & Target-Only \\
    \midrule
    Mean $\log(1+\mathrm{RE}_{1978})$ & 1.720     & 0.305        & $-0.580$      \\
    \bottomrule
  \end{tabular}
\end{table}

\section{Discussion}\label{sec: discussion}
In this paper, we introduced a transfer-learning methodology that permits drift in decision boundaries and applied it to optimal treatment-rule estimation, providing theoretical guarantees. When the source decision boundary differs from the target by a low-dimensional, smooth transform, our method effectively reduces a high-dimensional problem to a simple ERM, yielding faster rates than target-only SVM. In addition, we considered the application of this methodology in the estimation of individualized treatment rule estimation. 

Meanwhile, our methodology have several limitations. It would be challenging for the practitioner to know the type of transformation. We suggest the practitioner begin with simple transformations, such as translation and rotation, and then choose the one that preserves the smallest MSE on testing data. It would be equally interesting to explore even more complex transformations of the decision boundary. Also, we used a simple aggregation to avoid negative transfer. It would be useful to explore alternative aggregation strategies, such as convex aggregation. Lastly, it is also helpful to consider ITR estimation in observational studies, where the treatment assignment is unknown, such as \cite{zhao2019efficient, athey2021policy}.

\clearpage

\vspace*{-.4in}
\begin{center}
{\Large \textbf{Supplementary Material for ``Transfer Learning for Classification under Decision Rule Drift with Application to Optimal Individualized Treatment Rule Estimation''}}

\vspace{0.5cm}

{\large Xiaohan Wang$^1$, Yang Ning$^1$}

\vspace{0.3cm}

$^1$Department of Statistics and Data Science, Cornell University

\vspace{0.2cm}

\texttt{\{xw547, yn265\} @ cornell.edu}

\vspace{0.3cm}

\today
\end{center}

\thispagestyle{empty}
\pagebreak

\renewcommand{\thesection}{\Alph{section}}
\setcounter{section}{0}

\section{Proof of Theorem 1}
We start with the proof of Theorem 1 in the main document. The main idea can be summarized into three steps. In Section~\ref{sec: modifiedsvm}, we give a proof of the SVM estimator under our conditions, through techniques developed in \cite{steinwart2007fast}. In Section~\ref{sec: ermterm}, we utilize empirical process theory to analyze the ERM estimator. In Section~\ref{sec: transfer_est}, we prove our main result.
\subsection{Modified SVM Proof} \label{sec: modifiedsvm}
In this section, we focus on the proof of the rate of convergence for the SVM estimator on \(P-\)data, since the proof for \(Q-\)data is analogous. We denote \(f^*_P\) with \(f^*\), \(\gamma_P\) with \(\gamma\), and \(\alpha_P\) with \(\alpha\), if no confusion arises. 

\begin{theorem}\label{thm: svmmaintheorem}
  Let $X$ be the closed unit ball of $\mathbb{R}^d$, and $P$ be a distribution on $X \times Y \sim P$ with Tsybakov noise exponent $\alpha$ and modified geometric noise exponent $\gamma$. Let \(\gamma': = \gamma/d\), we define
$$
\beta:= \begin{cases}\frac{\gamma'}{2 \gamma'+1}, & \text { if } \gamma' \leq \frac{\alpha+2}{2 \alpha}, \\ \frac{2 \gamma'(\alpha+1)}{2 \gamma'(\alpha+2)+3 \alpha+4}, & \text { otherwise, }\end{cases}
$$
and $\lambda_n:=n^{-(\gamma'+d) / \gamma' \beta}$ and $\sigma_n:=n^{\beta /\gamma'}$ in both cases. Then for all $\varepsilon>0$ there exists a constant $C>0$ such that for all $x \geq 1$ and $n \geq 1$ the SVM without offset using the Gaussian RBF kernel $k_{\sigma_n}$ satisfies
$$
\operatorname{Pr}^*\left(T \in(X \times Y)^n: \mathcal{R}_P\left( f_{T, \lambda_n}\right) \leq \mathcal{R}_P+C x^2 n^{-\beta+\varepsilon}\right) \geq 1-e^{-x},
$$
where $\operatorname{Pr}^*$ denotes the outer probability of $P^n$ in order to avoid measurability considerations. If $\alpha=\infty$ the latter inequality holds if $\sigma_n=\sigma$ is a constant with $\sigma>2 \sqrt{d}$. 
\end{theorem}

\begin{remark}
  The result is obtained in an analogous manner as Theorem 2.8 of \cite{steinwart2007fast}, where the only difference lies in the coefficient \(\gamma\), since our definition of the noise condition is slightly different.
\end{remark}

\subsubsection{Approximation Error Term Proof}
In this section, we focus on the approximation error term and illustrate the difference under the modified geometric noise condition. We use the same notations as in \cite{steinwart2007fast}. With the covering number properly controlled, we can then consider the approximation error defined as:
\begin{align*}
a(\lambda) = \inf_{f\in \mathcal{H}}\lambda \left\| f \right\|_\mathcal{H}^2 + R_{l,P}(f)  - R_{l, P}^*.
\end{align*}

\begin{lemma}\label{lma: construction}
  Let \(X\) be a closed unit ball of \(\R^d\) and \(P\) be a probability measure on \(X\times Y\), where \(f^*(x)  \) is the Bayes estimator. On \(\acute{X} := 3X\) we have:
  \begin{align*}
    \acute{f}(x) = 
    \begin{cases}
      f^*(x) & \text{if } \left\| x \right\| \leq 1,\\
     f^*\left( \frac{x}{\left\| x \right\|} \right)    &otherwise.
    \end{cases}
  \end{align*}
  Consider the \(\acute{X}_{1} := \left\{ x \in \acute{X}; \acute{f}(x)  > 0 \right\},\acute{X}_{-1} := \left\{ x \in \acute{X}; \acute{f}(x) < 0 \right\} \), where \(X_1, X_{-1}\) is defined in a similar manner. Let \(B(x,r)\) denotes the open ball centered around \(x\) with radius \(r\) in \(\R^{d}\). Then for \(x \in X_1\) we have \(B(x, \tau(x)) \subset \acute{X}_1\), and for \(x \in X_{-1}\) we have \(B(x, \tau(x)) \subset \acute{X}_{-1}\).
\end{lemma}
\begin{proof}
The proof can be done in a similar manner as Lemma 4.1 of \cite{steinwart2007fast}. We skip the details here.
\end{proof}
\begin{remark}
  The above lemma shows that a small ball with respect to any point in the domain shall be of the same class as the center of the point. Such construction would free us from the cases where \(B(x, \tau(x))\not\subset X_1/X_{-1}\).
\end{remark}

Then, we can utilize the above result to control the approximation risk by considering a carefully constructed approximation. From \cite{steinwart2006explicit}, we can see that linear operator \(V_\sigma: L_2(\R^{d }) \to \mathcal{H}_\sigma(\R^{d})\) defined by:
\begin{align*}
  V_\sigma f(x)=\frac{(2 \sigma)^{d/ 2}}{\pi^{d/ 4}} \int_{\mathbb{R}^{d}} e^{-2 \sigma^2\|x-y\|_2^2} f(y) d y, \quad f \in L_2\left(\mathbb{R}^{d }\right), x \in \mathbb{R}^{d }
\end{align*}
 is an isomorphic isomorphism.
\begin{theorem}
  Let \(\sigma >0\), \(X\) be the closed unit ball of the Euclidean space \(\R^d\) and \(a_\sigma(\lambda)\) be the approximation error with respect to Gaussian RBF of bandwidth \(\sigma\). In addition, \(X\times Y \sim P\) has modified geometric noise coefficient \(\gamma\) with constant \(C\). Then, there's a coefficient \(c_d\) depending on the dimension \(d\) such that for all \(\lambda >0\), we have:
  \begin{align*}
    a_\sigma(\lambda) \leq c_d\left( \sigma^{d}\lambda + d^{\gamma /2}\sigma^{-\gamma} \right).
  \end{align*}
\end{theorem}
\begin{proof}
  The proof largely remains the same as the original SVM results, but the tools we were to use are slightly different.

  Consider a measurable function \(h_P(x): \mathbb{R}^{d}\to \R\) such that \(h_P(x):= \acute{f}_P(x) \), where \(\acute{f}_P(x) = 1\) if \(x \in \acute{X}_1 \) and \(\acute{f}_P(x)= -1\) if \(x \in \acute{X}_{-1} \), where \(\acute{X}_1, \acute{X}_{-1}\) are defined above. Then, we can consider the projection of \(h_P\) on \(\mathcal{H}_\sigma\). 

  Let \(g := (\sigma^2/\pi)^{d/4}h_P\), by \cite{zhang2004statistical}, we have:
  \begin{align*}
    \mathcal{R}_{l, P}(V_\sigma g)- \mathcal{R}_{l, P}^* = \E\left[ |2\eta - 1| | V_\sigma g - f_P^*| \right],
  \end{align*}
  where \(f_P^* := \text{sign}(2\eta(X) -1)\) is the Bayes classifier. It is easy to check that \(V_\sigma g \in [-1, 1]\) as well. When \(x\in X_1\), we may consider \(V_\sigma g\) :
  \begin{align*}
    V_\sigma g(x) & =\left(\frac{2 \sigma^2}{\pi}\right)^{d/2} \int_{\mathbb{R}^{d}} e^{-2 \sigma^2\|x-y\|_2^2} \acute{f}_P(y)  d y \\
    & =\left(\frac{2 \sigma^2}{\pi}\right)^{d/2} \int_{\mathbb{R}^{d}} e^{-2 \sigma^2\|x-y\|_2^2}\left(\acute{f}_P(y)+1\right) d y-1 \\
    & \geq\left(\frac{2 \sigma^2}{\pi}\right)^{d/2} \int_{B\left(x, \tau(x)\right)} e^{-2 \sigma^2\|x-y\|_2^2}\left(\acute{f}_P(y)+1\right) d y-1.
    \end{align*}
    Furthermore, from lemma \ref{lma: construction}, we have:
    \begin{align*}
      V_\sigma g(x) 
      &\geq 2\left(\frac{2 \sigma^2}{\pi}\right)^{d/2} \int_{B\left(x, \tau(x)\right)} e^{-2 \sigma^2\|x-y\|_2^2} d y-1\\
      & = 2\left(\frac{2 \sigma^2}{\pi}\right)^{d/2} \int_{B\left(0, \tau(x)\right)} e^{-2 \sigma^2\|y\|_2^2} d y -1.
    \end{align*}
Notice that \(e^{-\|x-y\|_2^2}\) is rotational invariant and \(\Gamma(x+1) = x\Gamma(x)\), we shall have:
\begin{align*}
  V_\sigma g(x) 
   & \geq \frac{2d}{\Gamma(1 + d/2)} \left({2 \sigma^2}\right)^{d/2}\int_0^{\tau(x)} e^{-2\sigma^22^2} 2^{d-1} dr -1\\
   &= \frac{4}{\Gamma(d/2)}  \int_0^{\sqrt{2}\sigma\tau(x)} e^{-r^2} 2^{d-1} dr -1\\
   & = \frac{2}{\Gamma(d/2)}\int_0^{2\tau(x)^2\sigma^2} e^{-r} r^{d/2 -1} dr -1,
\end{align*}
where we convert the integral to polar coordinate in the first line and conduct change of variable again in the last line. 
By symmetry, we have:
\begin{align*}
  1- V_\sigma g(x)  & \leq 2\left( 1-  \frac{1}{\Gamma(d/2)}\int_0^{2\tau(x)^2\sigma^2} e^{-r} r^{d/2 -1} dr\right)\\
  & = \frac{2}{\Gamma(d/2)}\int_{2\tau(x)^2\sigma^2}^\infty e^{-r} r^{d/2 -1} dr.
\end{align*}
Notice that we can do the same analysis for \(x \in X_{-1}\). Then, we have:
\begin{align*}
  |V_\sigma g(x) - f_P^* | \leq \frac{2}{\Gamma(d/2)}\int_{2\tau(x)^2\sigma^2}^\infty e^{-r} r^{d/2 -1} dr.
\end{align*}
As a result, we can consider the risk to be:
\begin{align*}
  \mathcal{R}_{l, P}(V_\sigma g) - \mathcal{R}^*_{l, P} 
  &= \int_X |2\eta_P(x) -1| \frac{2}{\Gamma(d/2)}\int_{2\tau(x)^2\sigma^2}^\infty e^{-r} r^{d/2 -1} dr dP_X\\
  & = \frac{2}{\Gamma(d/2)} \int_0^\infty e^{-r}r^{d/2 -1} \int_{|\tau(x)| \leq \left( \frac{r}{2\sigma^2} \right)^{1/2}} |2\eta_P(x) -1|dP_X dr\\
  & \leq \frac{r^{1-\gamma/2}c}{\sigma^{\gamma}\Gamma(d/2)}\int_0^\infty e^{-r} 2^{d/2-1 + \gamma/2}dr\\
  & = \frac{2^{1-\gamma/2}\Gamma((d+\gamma)/2)}{\sigma^{\gamma}\Gamma(d/2)}\\
  &\leq C d^{\gamma/2}\sigma^{-\gamma}.
\end{align*}
In the meantime, we can notice that: 
\begin{align*}
  \left\| g \right\|_{L_2}\leq \left( \frac{3^4 \sigma^2 }{\pi} \right)^{d/4} Vol(d),
\end{align*}
where \(Vol(d)\) would be the volume of a \(d-\)dimensional unit ball. It shall be easy to see that:
\begin{align*}
a_\sigma(\lambda) \leq c_d\left( \sigma^{d}\lambda + d^{\gamma /2}\sigma^{-\gamma} \right).
\end{align*}

\end{proof}
Notice that if we define \(\gamma':= \gamma/d\), the result shall be the same as what we have in the original modified SVM setting, where we can replace \(\gamma\) with \(\gamma'\). We can then follow a similar idea as in \cite{steinwart2007fast} to control the risk, and the theorem \ref{thm: svmmaintheorem} can be obtained.

\subsection{ERM term estimation}\label{sec: ermterm}

Notice that we can decompose the risk of the calibrated classifier by:
\begin{align}
  & \mathcal{R}_{Q}\left(  h(G_{\hat{f}_P}, \hat\theta) \right) - R_{Q}^{*}\notag
  \\ & = \mathcal{R}_Q\left(h(G_{\hat{f}_P}, \hat\theta)\right) - \inf_{\theta\in \Theta}\mathcal{R}_Q\left( h(G_{\hat{f}_P}, \theta) \right)\label{eq: transferfirst_term} \\
   & + \inf_{\theta\in \Theta}\mathcal{R}_Q\left(  h(G_{\hat{f}_P}, \theta)\right) - \mathcal{R}_Q\left(G_{f_Q^*}\right)\label{eq: transfersecond_term}.
\end{align} 
The symmetric difference between the sets \(G_{f_1}\) and \(G_{f_2}\) under the distribution \(Q\) shall be defined as:
\begin{align*}
d_{Q, \Delta}(f_1, f_2) = d_{Q, \Delta}(G_{f_1}, G_{f_2}) := \int \left( G_{f_1}\cap G_{f_2}^C \right)\cup \left( G_{f_1}^C \cap G_{f_2} \right)dQ.
\end{align*}

\subsubsection{Controlling the complexity}

In this section, we control the complexity of the space defined by the transformation.

\begin{lemma}\label{lma: funcomplexity}
Define \(\Gamma(\hat{f}, \Theta) := \left\{ {h(G_{\hat{f}}, \theta)}: \theta\in \Theta  \right\}\), then its \(\varepsilon-\)entropy can be controlled by:
\begin{align*}
     \mathcal{N}(\varepsilon M_1, \Gamma(\hat{f}_P, \theta), d_{\Delta}) \leq  \left(1+ \frac{4D}{\varepsilon} \right)^p,
\end{align*}
where \(D\) is the radius of \(\Theta\) and \(M_1\) is the Lipschitz constant given by assumption \ref{assu: boundaryassu}.
\end{lemma}

\begin{proof}
Notice that by Lipschitz condition and lemma \ref{lma: lipschitzcovering}, we'll have:
\begin{align*}
\mathcal{N}\left( \varepsilon M_1 , \Gamma(\hat{f}_P, \theta), d_{\Delta} \right) \leq \mathcal{N}\left( \varepsilon , \Theta, \left\| \cdot \right\| \right).
\end{align*}
Let \(B_D\) be a Euclidean ball of radius \(D\), we then have 
\begin{align*}
\mathcal{N}\left( \varepsilon , \Theta, \left\| \cdot \right\| \right) \leq \mathcal{N}\left( \varepsilon , B_D, \left\| \cdot \right\| \right).
\end{align*}

Notice that for \(B_D\), assume that \(x = \left\{ x_1, \dots, x_N \right\}\) is a maximal subset of \(B_D\) such that \(\left\| x_i - x_j \right\|\leq \varepsilon\). By maximality, \(x\) is an \(\varepsilon\)-net of \(B_D\). Then the open balls with center \(x_i\) and radius \(\varepsilon/2\) are disjoint and contained in \(B_{D + \varepsilon/2}\).
Comparing the volumes, we conclude that:
\begin{align*}
\mathcal{N}(\varepsilon, B_D, \left\| \cdot \right\|) \leq \left( 1 + \frac{4D}{\varepsilon} \right)^p.
\end{align*}
\end{proof}

With the complexity of the targeted function class properly controlled, we can now try to control the excess risk. In the following proof, we do so by a combination of maximal inequality and the margin condition.

\begin{theorem}\label{thm: ermmaintheorem}
Let \(Q\) be a distribution satisfying the margin condition with coefficient \(\alpha_Q\), and \(\hat{\theta}\) be the empirical risk minimizer, then the classifier $h(G_{\hat{f}_P}, \hat{\theta})$ satisfies: 
\begin{align*}
  \E\left[ \mathcal{R}_{Q}(h(G_{\hat{f}_P}, \hat{\theta})) - \mathcal{R}_{Q}(h(G_{\hat{f}_P}, \hat{\theta}^*)) \right] \leq C \left( \frac{p\log n_Q}{n_Q} \right)^{\frac{1+ \alpha_Q}{2+ \alpha_Q}}, 
\end{align*}
where \(\hat{\theta}^* = \inf_{\theta \in \Theta}R_Q(h(G_{\hat{f}_P}, \theta))\).
\end{theorem}

\begin{remark}
We note that the \(\log n_Q\) in the final presentation of the bound can actually be ``removed'' by restraining our parameter space into a countable subset.
\end{remark}
\begin{proof}
Our proof can be considered as an extension of the proof of \cite{kitagawa2018should} on VC {\it type} class, or an extension of \cite{tsybakov2004optimal}. In the following proof, we plan to use the peeling technique, and then utilize the maximal inequality to control the value of each partition. Using the controlled maximals, we apply the deviation empirical process inequality to control the tail probability, and finish the proof with the integral probability formula.

We can start by considering:
  \begin{align*}
    \mathcal{N}(\mathcal{G}, \left\| \cdot \right\|_{Q, 2}, \varepsilon \left\| \overline{F} \right\|_{Q,2}),
  \end{align*}
  where \(\mathcal{G} = \left\{ I\left\{ x\in h(G_{\hat{f}_P}, \theta) \right\} - I\{ x\in h(G_{\hat{f}_P}, \hat{\theta}^*)  \}: \theta\in \Theta\right\}\), and \(\overline{F} = 2(1 + DM)\) is the envelope.
 
  From lemma \ref{lma: funcomplexity}, we have:
  \begin{align*}
    \mathcal{N}(\mathcal{G}, \left\| \cdot \right\|_{Q, 2}, \varepsilon \left\| \overline{F} \right\|_{Q,2}) \leq \left( \frac{5}{\varepsilon} \right)^{2p}.
  \end{align*}
  This shows that the class of functions \(\mathcal{G}\) is a VC {\it type} class, as defined in \cite{chernozhukov2014gaussian}, with dimension \(2p\).     Since \(R_{Q,n}(h(G_{\hat{f}_P}, \hat{\theta})) - R_{Q,n}(h(G_{\hat{f}_P}, \hat{\theta}^* ))\leq 0\), we can notice that:
  \begin{align*}
    R_Q(h(G_{\hat{f}_P}, \hat{\theta})) - R_Q(h(G_{\hat{f}_P}, \hat{\theta}^* )) &\leq R_Q(h(G_{\hat{f}_P}, \hat{\theta})) - R_Q(h(G_{\hat{f}_P}, \hat{\theta}^* ))\\
    &  - \left(  R_{Q,n}(h(G_{\hat{f}_P}, \hat{\theta})) - R_{Q,n}(h(G_{\hat{f}_P}, \hat{\theta}^* )) \right),
  \end{align*}
  where we can consider the right-hand side as an empirical process. Denote \(g(\theta) = I\left\{ x\in h(G_{\hat{f}_P}, \theta) \right\} - I\{ x\in h(G_{\hat{f}_P}, \hat{\theta}^*)\} \). We have:
\begin{align*}
  R_Q(h(G_{\hat{f}_P}, \hat{\theta})) - R_Q(h(G_{\hat{f}_P}, \hat{\theta}^* )) &\leq  \left( \E_{Q}\left[ g\left( \hat{\theta} \right) \right] - \E_{Q,n}\left[ g\left( \hat{\theta} \right) \right] \right).
\end{align*}
Notice that for \(1/\sqrt{n_Q}\leq a<1\), whose value will be chose judiciously later, we have:
\begin{align*}
  R_Q(h(G_{\hat{f}_P}, \hat{\theta})) - R_Q(h(G_{\hat{f}_P}, \hat{\theta}^* )) &\leq V_a\left[ \E_Q\left[ g\left( \hat{\theta} \right) \right] + a^2 \right],
\end{align*}
where 
\begin{align*}
  V_a = \sup_{g\in \mathcal{G}} \left\{ \frac{\E_Q\left[ g  \right] - \E_{Q,n}\left[ g \right]}{E_Q\left[ g \right] + a^2} \right\} = \sup_{g\in \mathcal{G}} \left\{ \frac{\E_Q\left[ g  \right] - \E_{Q,n}\left[ g \right]}{E_Q\left[ g \right] + a^2} \right\}.
\end{align*}

If \(V_a\leq 1/2\), \(\E_Q\left[ g\left( \hat{\theta} \right) \right]\leq a^2\), so we can have:
\begin{align*}
  Q_n(g(\hat\theta)\geq a^2) \leq Q_n\left( V_a\geq 1/2 \right).
\end{align*}
Then, the idea is to bound \(Q_n\left( V_a\geq 1/2 \right)\), and then, by a judicious choice of \(a\), we can control the excess risk tail probability. To do so, we'll need to invoke deviation inequality for empirical processes. And this starts with considering \(\sup_{g\in \mathcal{G}} \E_Q\left[ g^2/\left( \E_Q[g] + a^2 \right) \right]\) so that we can apply the maximal inequality,
\begin{align*}
 \E_Q \left[ \left( \frac{g }{\E_Q[g] + a^2} \right)^2 \right] 
 & =  \frac{\E_Q \left[ g^2\right] }{\left(\E_Q[g] + a^2\right)^2  } \\
& \leq \frac{d_{Q, \Delta}( h(G_{\hat{f}_P}, \hat{\theta}^*), h(G_{\hat{f}_P}, \hat{\theta}))}{(\E_Q[g] + a^2)^2}\\
& \leq C_1 \frac{\left( \E_Q[g] \right)^{\alpha_Q/(1+\alpha_Q)}}{(\E_Q[g] + a^2)^2},
\end{align*}
where the first inequality is a result of lemma \ref{lma: tsybakovlemma2}. Going further, since \(\E_Q[g] \geq 0\), we shall have:
\begin{align*}
  \sup_{g\in \mathcal{G}}  \E_Q \left[ \left( \frac{g }{\E_Q[g] + a^2} \right)^2 \right]  & \leq C_1 \sup_{x \geq 0} \frac{x^{2\alpha_Q/(1+\alpha_Q)}}{(x^2 + a^2)^2} \leq C_1 \frac{1}{a^4}a^{2\alpha_Q/(1+\alpha_Q)} = C_1 a^{\frac{2\alpha_Q}{1+\alpha_Q} -4}.
\end{align*}

Since \(\mathcal{G}\) is pointwise measurable and the envelope function is \(\overline{F} = 2(1+DM_1)/a^2\), with Theorem 5.1 of \cite{chernozhukov2014gaussian}, with probability at least \(1-t^{-q/2}\), we have:
\begin{align*}
  V_a \leq \left( 1 + \beta \right) \E_Q[V_a] + K(q) \left( \sqrt{\frac{C_1 a^{\frac{2\alpha_Q}{1+\alpha_Q} -4} t}{n_Q}} + \frac{\sqrt{t}}{n_Q a^2} \right) + \frac{t}{\beta n_Q a^2},
\end{align*}
where \(q\), \(\beta\), and \(t\) are some constant that will be chose later.

To utilize the margin condition, we can consider partitioning the function class and use a peeling-like method to obtain a sharp bound. Let \(r > 1\) be a constant, we can partition the function class into \(\mathcal{G}_0, \mathcal{G}_1, \dots \mathcal{G}_j, \dots \) where for \(i = 0\), \(\mathcal{G}_0 = \left\{ g \in \Gamma(\hat{f}_P, \theta):  \E[g] \leq a^2 \right\}\), and \(\mathcal{G}_j = \left\{ g \in \Gamma(\hat{f}_P, \theta): 2^{2(j-1)}a^2 \leq  \E[g] \leq 2^{2j}a^2 \right\}\), for \(j = 1, 2, \dots \).
\begin{align*}
  \E[V_a] & \leq \sup_{g \in \mathcal{G}_0} |V_a| + \sum_{j=1}^\infty\sup_{g \in \mathcal{G}_j} |V_a|\\
  & \leq \frac{1}{a^2} \left[ \sup_{g \in \mathcal{G}_0}\left| \E_{Q, n}(g)- E_Q(g) \right| + \sum_{j=1}^\infty (1+2^{2(j-1)})^{-1}\sup_{g\in \mathcal{G}_i}|  \E_{Q, n}(g)- E_Q(g) |\right].
\end{align*}

Then, for each partition, we can see that 
\begin{align*}
  \sup_{g \in \mathcal{G}_j} \E_Q\left[ g^2 \right] &\leq d_{Q, \Delta}(h(G_{\hat{f}_P}, \theta), h(G_{\hat{f}_P}, \hat\theta^*)) \leq C_1\left( \E_Q[g]  \right)^{\frac{\alpha_Q}{1+\alpha_Q}}\leq C_1 2^{2 j\alpha_Q/(1+\alpha_Q)}a^{2 \alpha_Q/(1+\alpha_Q)}\\
   \overline{F}_j & =  2^{j\frac{\alpha_Q}{1+\alpha_Q}+1}(1+DW).
\end{align*}

Note that the partitioned space are still VC {\it type} class by subset preservation property and homogeneity. In addition, we ``inflated'' our envelope function for each partition so as to apply the peeling technique. Then, for each of the partitioned spaces, we may apply corollary 5.1 of \cite{chernozhukov2014gaussian}, which would give us:
\begin{align*}
  \E[V_a] & \lesssim a^{\frac{\alpha_Q}{1+\alpha_Q}-2}\sqrt{\log \frac{10(1+DM)}{a^{\frac{\alpha_Q}{1+\alpha_Q}}} } \sqrt{\frac{2p}{n_Q}} +\log\left( \frac{10(1+DM)}{a^{ \frac{\alpha_Q}{1+\alpha_Q}}}\right)\frac{2p}{n_Q a^2}\\
  & + \sum_{j= 1}^\infty \frac{2^{\frac{\alpha_Q}{1+\alpha_Q}j}}{1+2^{2(j-1)}} a^{\frac{\alpha_Q}{1+\alpha_Q}-2}\sqrt{\log \frac{10(1+DM)}{\sqrt{C_1}a^{\frac{\alpha_Q}{1+\alpha_Q}}} } \sqrt{\frac{2p}{n_Q}} \\
  & + \sum_{j= 1}^\infty \frac{1}{1+ 2^{2(j-1)}}\log\left( \frac{10(1+DM)}{\sqrt{C_1}a^{ \frac{\alpha_Q}{1+\alpha_Q}}}\right)\frac{2p}{n_Q a^2}\\
  & \lesssim a^{\frac{\alpha_Q}{1+\alpha_Q}-2}\sqrt{\log \frac{10(1+DM)}{a^{\frac{\alpha_Q}{1+\alpha_Q}}} } \sqrt{\frac{2p}{n_Q}} +\log\left( \frac{10(1+DM)}{a^{ \frac{\alpha_Q}{1+\alpha_Q}}}\right)\frac{2p}{n_Q a^2}\\
  & +  a^{\frac{\alpha_Q}{1+\alpha_Q}-2}\sqrt{\log \frac{10(1+DM)}{a^{\frac{\alpha_Q}{1+\alpha_Q}}} } \sqrt{\frac{2p}{n_Q}} \frac{4}{1-2^{\frac{\alpha_Q}{1+\alpha_Q}-2}} \\
  & + a^{-2} \log \frac{10(1+DM)}{a^{\frac{\alpha_Q}{1+\alpha_Q}}} \frac{2p}{n_Q} \frac{4}{1-2^{-2}}\\
  &\lesssim a^{\frac{\alpha_Q}{1+\alpha_Q}-2}\sqrt{\log \frac{10(1+DM)}{a^{\frac{\alpha_Q}{1+\alpha_Q}}} } \sqrt{\frac{2p}{n_Q}} + a^{-2}\log\left( \frac{10(1+DM)}{a^{ \frac{\alpha_Q}{1+\alpha_Q}}}\right)\frac{2p}{n_Q}.
\end{align*}
Notice that since \(1/\sqrt{n_Q}\leq a <1, \alpha_Q\geq 0 \), we have:
\begin{align*}
  \E[V_a]
  & \lesssim  a^{\frac{\alpha_Q}{1+\alpha_Q}-2}\sqrt{\log n_Q} \sqrt{\frac{2p}{n_Q}} + a^{-2}\log n_Q\frac{2p}{n_Q}.
\end{align*}
In addition, for \(n_Q \geq 2p\log n_Q\), we have:
\begin{align*}
\E[V_a]
& \lesssim  a^{\frac{\alpha_Q}{1+\alpha_Q}-2}\sqrt{\frac{2p\log n_Q}{n_Q}}.
\end{align*}
Furthermore, the bound trivially holds if \(n_Q < 2p\log n_Q\). Now, we can proceed with the probability bound. Since if otherwise, the declared result shall hold instantly. For \(n_Q \geq  2p\log n_Q\), with probability greater than \(1- t^{-q}\), we have: 

\begin{align}
  V_a \leq C_3 \left(  a^{\frac{\alpha_Q}{1+\alpha_Q}-2}\sqrt{\frac{2p \log n_Q}{n_Q}} + K(q) \left( a^{\frac{\alpha_Q}{1+\alpha_Q}-2}\sqrt{\frac{ t}{n_Q}} + \frac{\sqrt{t}}{n_Q a^2} \right) + \frac{t}{ n_Q a^2} \right)\label{eq: VaInitialBound},
\end{align}
where \(C_3 = (8C_2 C_1 \log10(1+DM)\frac{\alpha_Q}{\alpha_Q+1} \vee 1)\) and \(C_2\) is the constant from corollary 5.1 of \cite{chernozhukov2014gaussian}.

Considering the solution to \(C_3   a^{\frac{\alpha_Q}{1+\alpha_Q}-2}\sqrt{\frac{2p\log n_Q}{n_Q}}  = 1\) with respect to \(a\) be \(\varepsilon_n\), we can see that:
\begin{align*}
\left( C_3\sqrt{\frac{2p\log n_Q}{n_Q}} \right)^{\frac{1+\alpha_Q}{2+ \alpha_Q}}= \varepsilon_{n}.
\end{align*}

We may now consider \(a \geq \varepsilon_n\) as a reasonable bound. Define \(\frac{1}{kt} = \left( \frac{\varepsilon_{n}}{a} \right)^{2}\) for \(t\geq 1\), we can notice that:
\begin{align*}
  C_3 \left( a^{\frac{\alpha_Q}{1+\alpha_Q}-2}\sqrt{\frac{2p\log n_Q}{n_Q}} \right) &\leq C_3 \left(   a^{-1 -\frac{1}{1+\alpha_Q}}\sqrt{\frac{2p\log n_Q}{n_Q}} \right)\\
  &\leq \frac{\left(C_3\sqrt{\frac{2p\log n_Q}{n_Q}} \right)^{\frac{1+\alpha_Q}{2+ \alpha_Q}}}{a} \left(\frac{\left(C_3\sqrt{\frac{2p\log n_Q}{n_Q}} \right)^{\frac{1+\alpha_Q}{2+ \alpha_Q}}}{a} \right)^{\frac{1}{1+\alpha_Q}} \\
  & \leq \frac{\varepsilon_{n}}{a}\\
  a^{\frac{2\alpha_Q}{1+\alpha_Q}-2} = \left( a^{- \frac{2+\alpha_Q}{1+\alpha_Q}} \right)^2 a^2 &\leq \left( \varepsilon_n^{- \frac{2+\alpha_Q}{1+\alpha_Q}} \right)^2 \varepsilon_n^2  \leq \frac{n_Q\varepsilon_{n}^2}{ 2p C_3^2}.
\end{align*}

Since the right-hand side of \eqref{eq: VaInitialBound} is monotonically decreasing with respect to \(a\). We can then consider replacing \(a\) with \(\varepsilon_n\), and then the bound will then be:
\begin{align*}
  V_q &\leq C_3 \left( \frac{1}{\sqrt{k}} + K(q)\left( \sqrt{\frac{1}{2p C_3^2 k}} + \sqrt{\frac{1}{k n_Q\varepsilon_{n}^2}}\sqrt{\frac{1}{n_Q\varepsilon_{n}^2}} \right) + \frac{1}{kn_Q\varepsilon_{n}^2}\right).
\end{align*}
Recall that \(n_Q\varepsilon_{n}^2\geq 1\), we then have:
\begin{align*}
  V_q &\leq C_2 \left( \frac{1}{k} + K(q)\left( \sqrt{\frac{1}{2p C_3^2 k}} + \sqrt{\frac{1}{k}}\right)) + \frac{1}{k}\right).
\end{align*}
By choosing large enough \(k\), we have:
\begin{align*}
  Q\left( V_a \leq \frac{1}{2} \right)\geq 1- t^{-q}.
\end{align*}

That is:
\begin{align*}
  Q_n\left(  g(\hat\theta)\geq kt \varepsilon_{n}^2\right) \leq t^{-q}.
\end{align*}

Then, we may have:
\begin{align*}
  R_Q(h(G_{\hat{f}_P}, \hat{\theta})) - R_Q(h(G_{\hat{f}_P}, \hat{\theta}^* )) &= \int Q_n\left(  g(\hat\theta)\geq t\right) dt\\
  & = \int_0^{k\varepsilon_{n}^2} Q_n\left(  g(\hat\theta)\geq t\right) dt + \int_{k\varepsilon_{n}^2}^\infty Q_n\left(  g(\hat\theta)\geq t\right) dt\\
  & \leq k \varepsilon_{n}^2 \left( 1+\frac{1}{q-1} \right)\\
  & \leq k(1+\frac{1}{q-1}) C_3^{\frac{2(1+\alpha_Q)}{2+\alpha_Q}}\left( \frac{2p\log n_Q}{n_Q} \right)^{\frac{1+\alpha_Q}{2+\alpha_Q}}.
\end{align*}

By choosing \(q>2\), the claimed result is proved.

\end{proof}

\subsection{Transfer Estimation} \label{sec: transfer_est}

\begin{theorem}\label{thm: ftilderate}
Let \(G_{\tilde{f}_Q}\) be the calibrated classifier. We define \(\gamma_P' = \gamma_P/d\), \(\gamma_Q' = \gamma_Q/d\),  
\begin{align*}
  \beta_P := \begin{cases}
    \frac{\gamma_P'}{2\gamma_P'+1} & \text{ if } \gamma_P'/d \leq \frac{\alpha_P  + 2 }{2 \alpha_P}\\
    \frac{2\gamma_P'(\alpha_P+1)/d}{2\gamma_P'(\alpha_P+2)/d + 3 \alpha_P + 4}&\text{ otherwise, }
  \end{cases}
\end{align*}
\(\beta_Q\) is defined in the same way, \(\lambda_{n_P, P}= n_P^{-(\gamma_P'+d)/\gamma_P'\beta_P}\) and \(\sigma_{n_P, P}= n_P^{\beta_P/\gamma_P'}\). Then, uniformly over $(P,Q)\in \Pi$ and for any \(\varepsilon_P > 0\), we have
\begin{align*}
  \mathcal{R}_{Q}(G_{\tilde{f}_Q})  -  \mathcal{R}_{Q}^*
  = O_P \left( \left( \frac{p\log n_Q}{n_Q} \right)^{\frac{1+ \alpha_Q}{2+ \alpha_Q}} +  \delta^{\gamma_Q}  + \left(  n_P^{-\beta_P+\varepsilon_P}\right)^{\frac{\alpha_P}{1+\alpha_P}} \right).
\end{align*}
\end{theorem}

\begin{proof}
We start by decomposing the risk, and then deal with each component respectively, 
\begin{align}
\eqref{eq: transfersecond_term} &= \inf_{\theta\in\Theta}\mathcal{R}_{Q}\left(  h(G_{\hat{f}_P}, {\theta})  \right) -\mathcal{R}_{Q}\left(  h(G_{{f}_P^*}, {\theta}^*)  \right)\label{eq: transfer_secondfirst}\\
&+  \mathcal{R}_{Q}\left(  h(G_{{f}_P^*}, {\theta}^*)  \right) - \mathcal{R}_{Q}\left(  G_{f_Q^*}  \right)\label{eq: transfer_22}
\end{align}
Then, by lemma \ref{lma: tsybakovlemma2}, assumption \ref{assu: boundaryassu} and strong density assumption, we shall have:
\begin{align*}
\eqref{eq: transfer_secondfirst} &=
\inf_{\theta\in\Theta}\mathcal{R}_{Q}\left(  h(G_{\hat{f}_P}, {\theta})  \right) -\mathcal{R}_{Q}\left(  h(G_{{f}_P^*}, {\theta}^*)  \right) \\
& \leq \mathcal{R}_{Q, \Delta}\left(  h(G_{\hat{f}_P}, {\theta}^*)  \right) - \mathcal{R}_{Q, \Delta}\left(h\left(  G_{f_P^*}, \theta^* \right)  \right)\\
& \lesssim \mathcal{R}_{ \Delta}\left(  h(G_{\hat{f}_P}, {\theta}^*)  \right) - \mathcal{R}_{ \Delta}\left( \left( h(G_{ {f}_P^*},\theta^*)   \right)\right)\\
& \lesssim \mathcal{R}_{ \Delta}\left(  G_{\hat{f}_P}  \right) - \mathcal{R}_{ \Delta} \left( G_{ {f}_P^*}\right)\\
 &  \lesssim \left( \mathcal{R}_{P}\left(  G_{\hat{f}_P}  \right) - \mathcal{R}_{P}\left(  G_{ {f}_P^*}\right)  \right)^{\frac{\alpha_P}{1+\alpha_P}}.
\end{align*}
Notice that for the term inside the parenthesis, we can utilize Theorem \ref{thm: svmmaintheorem} to obtain the rate.

For \eqref{eq: transfer_22}, we can simply restrict the bias to the set of points such that are close to the decision boundary, measured by \(\delta\). By lemma \ref{lma: tauandhausdorff}, we can show that, if \(\tau(x) > \delta\), \(I\{ x\in G_{f_Q^*} \} = I\left\{ x\in h( G_{f_P^*}, \theta^*) \right\} \). That is, we can start by noticing that 
\begin{align*}
\mathcal{R}_{Q}\left(  h(G_{{f}_P^*}, {\theta}^*)  \right) - \mathcal{R}_{Q}\left(  G_{f_Q^*}  \right) & = \int_{\tau(x) \leq \delta} |2\eta_Q(x) -1| dQ \\
  &\leq \delta^{\gamma_Q},
\end{align*}
where the last inequality is a result of the modified geometric noise condition.

Putting everything together, we then have:
\begin{align*}
&\phantom{A}\inf_{\theta\in \Theta}\mathcal{R}_Q\left(  h(G_{\hat{f}_P}, \theta)  \right) - \mathcal{R}_Q\left(G_{f_Q^*}   \right)\lesssim \left( \mathcal{R}_{P}\left(  G_{\hat{f}_P}  \right) - \mathcal{R}_{P}\left(  G_{ {f}_P^*}\right)  \right)^{\frac{\alpha_P}{1+\alpha_P}} +  \delta^{\gamma_Q}.
\end{align*}
By Theorems \ref{thm: svmmaintheorem} and \ref{thm: ermmaintheorem}, we can see that 
$$
  \mathcal{R}_{Q}(G_{\tilde{f}_Q})  -  \mathcal{R}_{Q}^*
  = O_P \left( \left( \frac{p\log n_Q}{n_Q} \right)^{\frac{1+ \alpha_Q}{2+ \alpha_Q}} +  \delta^{\gamma_Q}  + \left(  n_P^{-\beta_P+\varepsilon_P}\right)^{\frac{\alpha_P}{1+\alpha_P}} \right).   
$$
\end{proof}

{\bf Proof of Theorem 1 in the main paper:}
\begin{proof}
We notice that \(h(G_{\hat f_P}, 0) = G_{\hat f_P}\), then
\begin{align*}
\mathcal{R}_{Q}(G_{\hat f_P})  -  \mathcal{R}_{Q}^* 
&= \mathcal{R}_{Q}\left(  h(G_{\hat{f}_P}, 0)  \right) - \mathcal{R}_{Q}\left(  h(G_{\hat{f}_P}, \theta^*)  \right) + \mathcal{R}_{Q}\left(  h(G_{\hat{f}_P}, \theta^*)  \right) -\mathcal{R}_{Q}\left(  h(G_{{f}_P^*}, {\theta}^*)  \right)  
\\& +\mathcal{R}_{Q}\left(  h(G_{{f}_P^*}, {\theta}^*)  \right) - \mathcal{R}_{Q}\left(  G_{f_Q^*}  \right).
\end{align*}
Using a similar approach as the proof to Theorem \ref{thm: ftilderate}, we shall see that:
\begin{align*}
\mathcal{R}_{Q}(\hat{f_P})  -  \mathcal{R}_{Q}^* &\lesssim \left\| \theta^* \right\|_2+  \delta^{\gamma_Q}+ \left(  n_P^{-\beta_P+\varepsilon_P}\right)^{\frac{\alpha_P}{1+\alpha_P}}.
\end{align*}
The rest of the proof is the same as Proposition 12 of \cite{reeve2021adaptive}.
\end{proof}

\clearpage
\section{Auxiliary Results}
\begin{lemma}\label{lma: lipschitzcovering}
Let \((\Omega_1, d_1), (\Omega_2, d_2)\) be a metric space, while \(h: (\Omega_1, d_1) \to (\Omega_2, d_2)\) be a \(L\)-Lipschitz function. Then, if \(X_\varepsilon\) is a \(\varepsilon\)-net defined on \((\Omega_1, d_1)\), \(h\left( X_\varepsilon \right)\) is a \(L\varepsilon-\)net on \(h(\Omega_1)\).
\end{lemma}
\begin{proof}
Notice that in \((\Omega_1, d_1)\) for any \(\varepsilon >0\), there exists an \(\varepsilon-\)net \(\mathcal{X}_\varepsilon\) such that \(\forall x \in (\Omega, d_1)\), \(\exists x_j \in \mathcal{X}_\varepsilon \) such that 
\begin{align}
    d_1(x, x_j) \leq \varepsilon.
\end{align}

Then, the same phenomenon exists for \((\Omega_2, d_2)\) as well. For any \(\varepsilon\), we can consider the same \(\varepsilon-\)net, while the set in \(\Omega_2\) will be \(h(X_\varepsilon)\), Then, we may have, for any \(y \in h(\Omega_2)\), then for \(x_j\in X_\varepsilon\),
\begin{align*}
    d_2(y, h(x_j)) &= d_2\left( h\left( h^{-1}\left( y \right) \right), h(x_j) \right)\\
    &\leq L d_1\left( h^{-1}\left( y \right) , x_j \right).
\end{align*}
Notice that \(h^{-1}(y)\in \Omega_1\). As a result, 
For any \(y \in h(\Omega_2)\),  \(\exists \ h(x_j)\in h(X_\varepsilon)\),
\begin{align*}
    d_2(y, h(x_j)) \leq L\varepsilon.
\end{align*}
\end{proof}

This lemma controls the entropy number of a Lipschitz-transformed space. In our case, we shall see that if \(h(f, \theta)\) is Lipschitz, so will  \(x - h(f, \theta) +c\), making it much easier for us to control the complexity.

\begin{lemma}\label{lma: tsybakovlemma2}
When \(P\) satisfies the margin condition with coefficient \(\alpha\), for any measurable functions \(f_1\) and \(f_2\), we shall have:
  \begin{align*}
     \left(  \mathcal{R}_{P, \Delta}(f_1) - \mathcal{R}_{P, \Delta}(f_2) \right)^{(1+\alpha)/\alpha} \leq \mathcal{R}_P(f_1) - \mathcal{R}_P(f_2) \leq \mathcal{R}_{P, \Delta}(f_1) - \mathcal{R}_{P, \Delta}(f_2). 
  \end{align*}
\end{lemma}
\begin{proof}
See lemma 2 of \cite{mammen1999smooth}.
\end{proof}

\begin{lemma}\label{lma: tauandhausdorff}
For two closed sets \(A\) and \(B\), such that \(A \subseteq \Omega, B\subseteq \Omega \), for any point \(x \in \Omega\), we have:
\begin{align*}
  \left\{ \tau_A(x) >  (d_H(A, B) \vee d_H(\overline{A^C}, \overline{B^C})) \right\} \cap A \subseteq A\cap B,
\end{align*}
where \(\tau_A(x) : = \min_{y \in \partial A}d(x, y)\), \(\partial A\) is the boundary of \(A\), \(d_H\) is the Hausdorff distance, and \(\overline{A^C}, \overline{B^C}\) represents the closure of the complement of \(A\) and \(B\) with respect to \(\Omega\) respectively.
\end{lemma}

\begin{proof}
We prove the result by showing that \(A \cap B^C \subseteq  \left\{ \tau_A(x) \leq  (d_H(A, B) \vee d_H(\overline{A^C}, \overline{B^C})) \right\} \cap A\). We can see that for any \(z \in A\cap B^C\): 
\begin{align*}
  \tau_A(z) 
  & = \min_{y \in A}d(z, y) \vee \min_{y \in \overline{A^C}}d(z, y)\\
  & = \min_{y \in \overline{A^C}}d(z, y)\\
  & \leq \sup_{z \in A \cap B^C}   \min_{ y \in \overline{A^C}}d(z, y)\\
  & \leq d_H(A\cap B^C, \overline{A^C})\\
  & \leq d_H(\overline{B^C}, \overline{A^C}),
\end{align*}
That is:
\begin{align*}
  A \cap B^C \subseteq \left\{ \tau_A(x) \leq  (d_H(A, B) \vee d_H(\overline{A^C}, \overline{B^C})) \right\} \cap A.
\end{align*}
Our claim is thus obvious.
\end{proof}

\clearpage
\section{Additional Simulations}
In this section, we present the additional simulations which we conducted to validate the performance of our method.

\subsection{Simulation Settings}
The simulations are organized into 20 categories based on the decision boundary type (linear or nonlinear), transformation type (translation or rotation), and the type of regression function. In addition to the settings presented in the main document, we present the remaining settings here. Both source and target data share the same marginal distribution \(P_X = Q_X\), uniformly distributed on \([-3,3]^d\), where we vary the dimension. For decision boundary, we consider the following:
\begin{enumerate}
  \item Linear decision boundary: The decision boundary is defined as:
  \[
    G_P^* = \{\beta^\top x > 0\},
  \]
  \item Nonlinear decision boundary:
  \[
    G_P^* = \{x^\top Q x > 0\}, 
       \]
    where matrices \(Q\) is defined as
       \[
    Q = \begin{pmatrix}
           0.3 & 0 & 0 & & \\
           0 & 0 & \frac{1}{2} & & \\
           0 & \frac{1}{2} & 0 & & \\
           & & & \ddots & \\
           & & & & 0
       \end{pmatrix}.
       \]
\end{enumerate}
For decision boundary, we consider the following:
\begin{enumerate}
  \item Translation: For example, for the linear decision boundary, the decision boundary for the target distribution of translation is defined as:
 \[
    G_Q^* = \{\beta^\top x + \theta > 0\},
  \]
  where \(\theta\) is the parameter for transformation. And all other transformed decision boundaries are defined in a likewise manner. 
  \item Rotation: Likewise, the linear decision boundary, the decision boundary for the target distribution of rotation is defined as:
  \[G_Q^* = \{ \beta^\top P x > 0\},\]
  where \(P\) is defined as:
  \[
P = \begin{pmatrix}
      \cos(\theta) & -\sin(\theta) & & & \\
      \sin(\theta) & \cos(\theta) & & & \\
      & & 1 & & \\
      & & & \ddots & \\
      & & & & 1
  \end{pmatrix}.
  \]
  All other rotations are defined in a similar manner.
\end{enumerate}

For each category, we examined different regression functions, with the regression function for the source distribution of linear decision boundary as an example.
\begin{enumerate}
  \item Deterministic: 
  \begin{align*}
    \eta_P = \begin{cases}
      1 & \text{ if } \beta^\top x > 0\\
      0 & \text{ otherwise }
    \end{cases}
  \end{align*}
  \item Linear: Linear regression function:
  \begin{align*}
    \eta_P = 1/2 \frac{1}{C} \beta^\top x,
  \end{align*}
  where \(C\) is a normalizing constant depending on the marginal distribution.
  \item Logistic: Logistic regression function:
  \begin{align*}
    \eta_P =  \frac{\exp(\beta^\top x)}{1+ \exp(\beta^\top x)},
  \end{align*}
  \item Truncate: Truncated linear regression function:
  \begin{align*}
    \eta_P =   1/2 + \textrm{sign}\left(\beta^\top x \right)\max\left(\left|\frac{1}{C} \beta^\top x\right|, 0.1\right)
  \end{align*}
  \item Truncatelogit: Truncated logistic regression function:
  \begin{align*}
    \eta_P =  1/2 + \textrm{sign}\left( \beta^\top x \right)\max\left(\left|\frac{\exp(\beta^\top x)}{1+ \exp(\beta^\top x)} - \frac{1}{2}\right|, 0.1\right)
  \end{align*}
\end{enumerate}

For each setting, we varied three parameters:
\begin{enumerate}
  \item Dimension: The dimensionality of the feature space. We consider \(d = \left\{ 3, 5, 10, 15, 20, 30 \right\}\)
  \item Shift: The magnitude of the transformation. 
  \item Shares: The ratio of source to target data. We consider \(shares = \left\{ 2, 5, 8, 16, 32, 64 \right\}\).
\end{enumerate}

\subsection{Results}

\begin{figure}[p]
  \centering
  \includegraphics[ width=\textwidth]{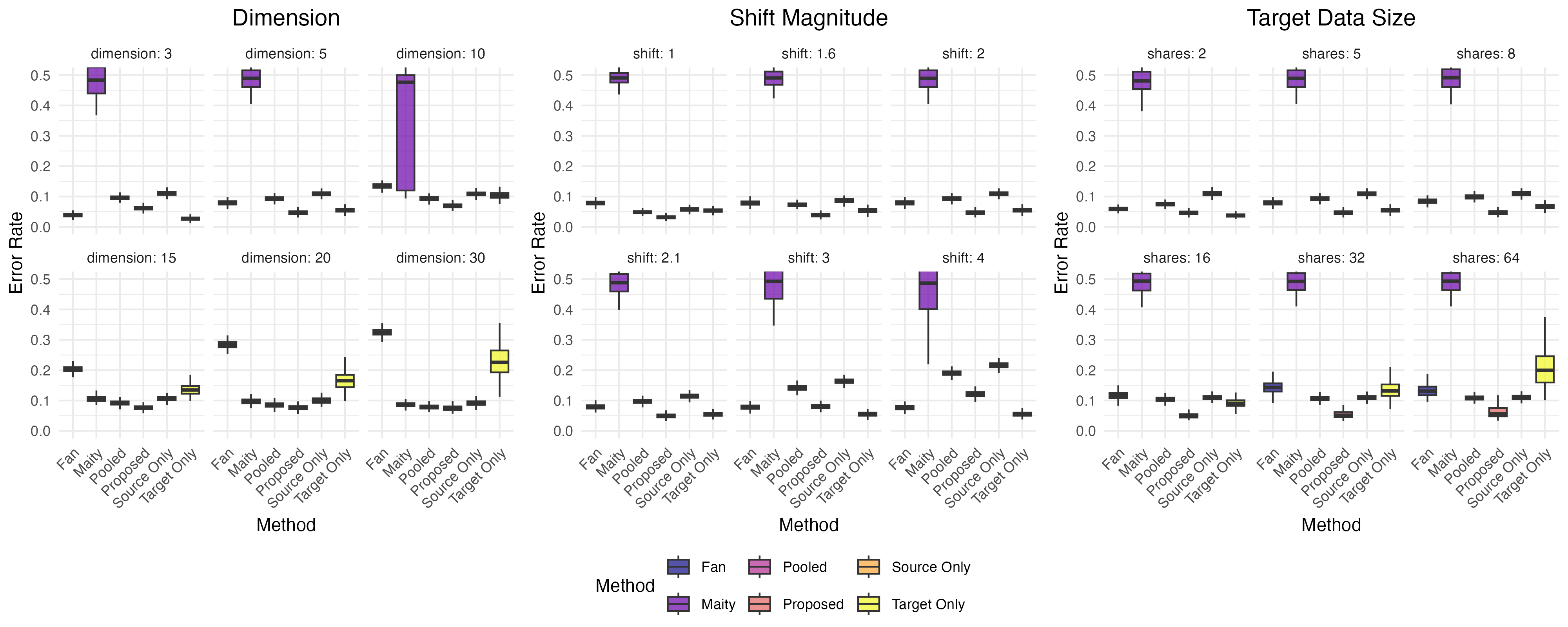}
  \caption{Misclassification rates under different dimension, shift magnitude, and data sizes for linear decision boundary with translation transformation under deterministic regression function. Six methods as specified in the main document are compared.}
\end{figure}

\begin{figure}[p]
  \centering
  \includegraphics[ width=\textwidth]{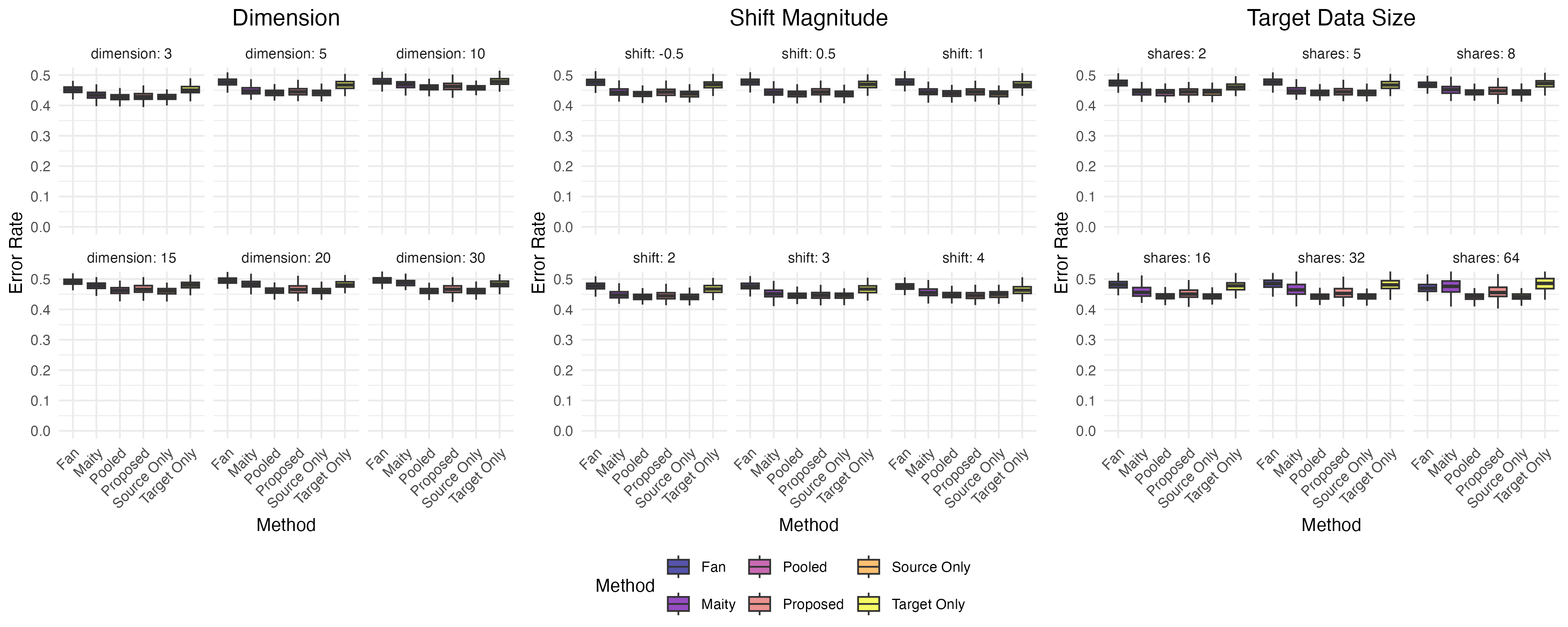}
  \caption{Misclassification rates under different dimension, shift magnitude, and data sizes for linear decision boundary with translation transformation under linear regression function. Six methods as specified in the main document are compared.}
\end{figure}

\begin{figure}[p]
  \centering
  \includegraphics[ width=\textwidth]{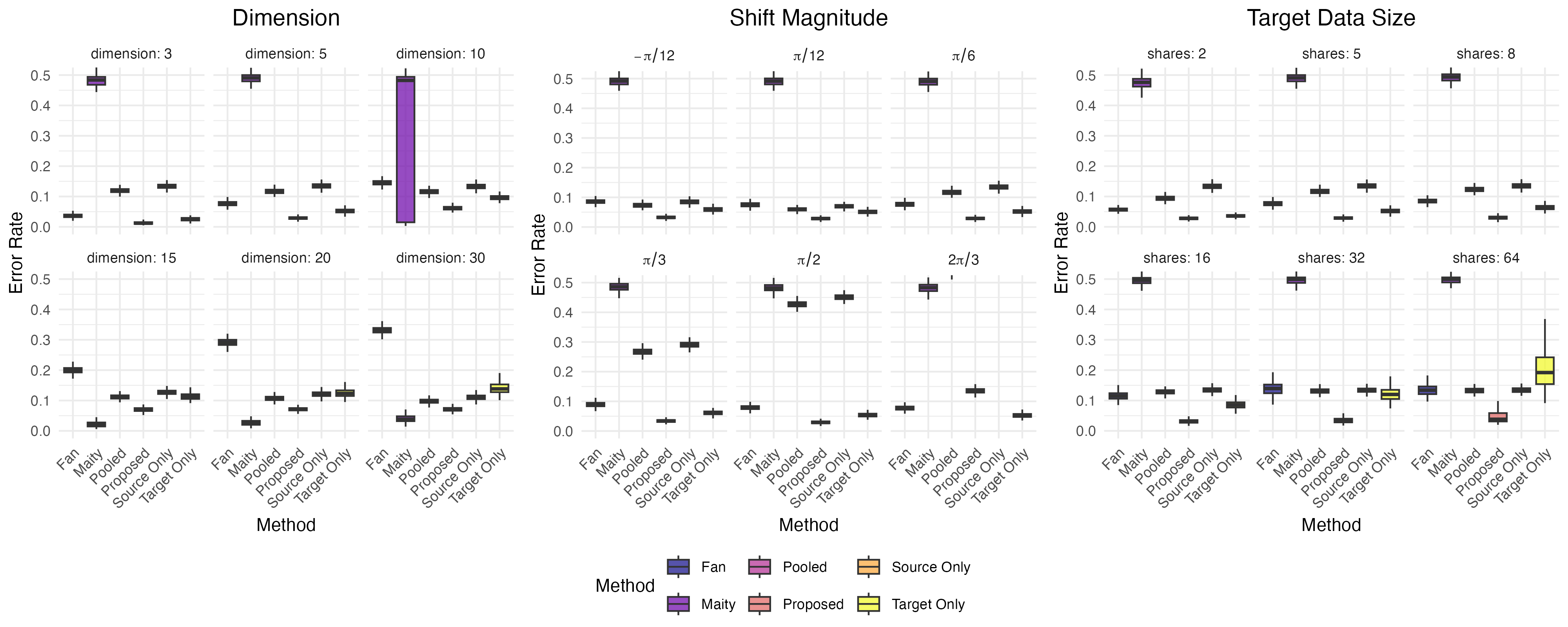}
  \caption{Misclassification rates under different dimension, shift magnitude, and data sizes for linear decision boundary with rotation transformation under deterministic regression function. Six methods as specified in the main document are compared.}
\end{figure}

\begin{figure}[p]
  \centering
  \includegraphics[ width=\textwidth]{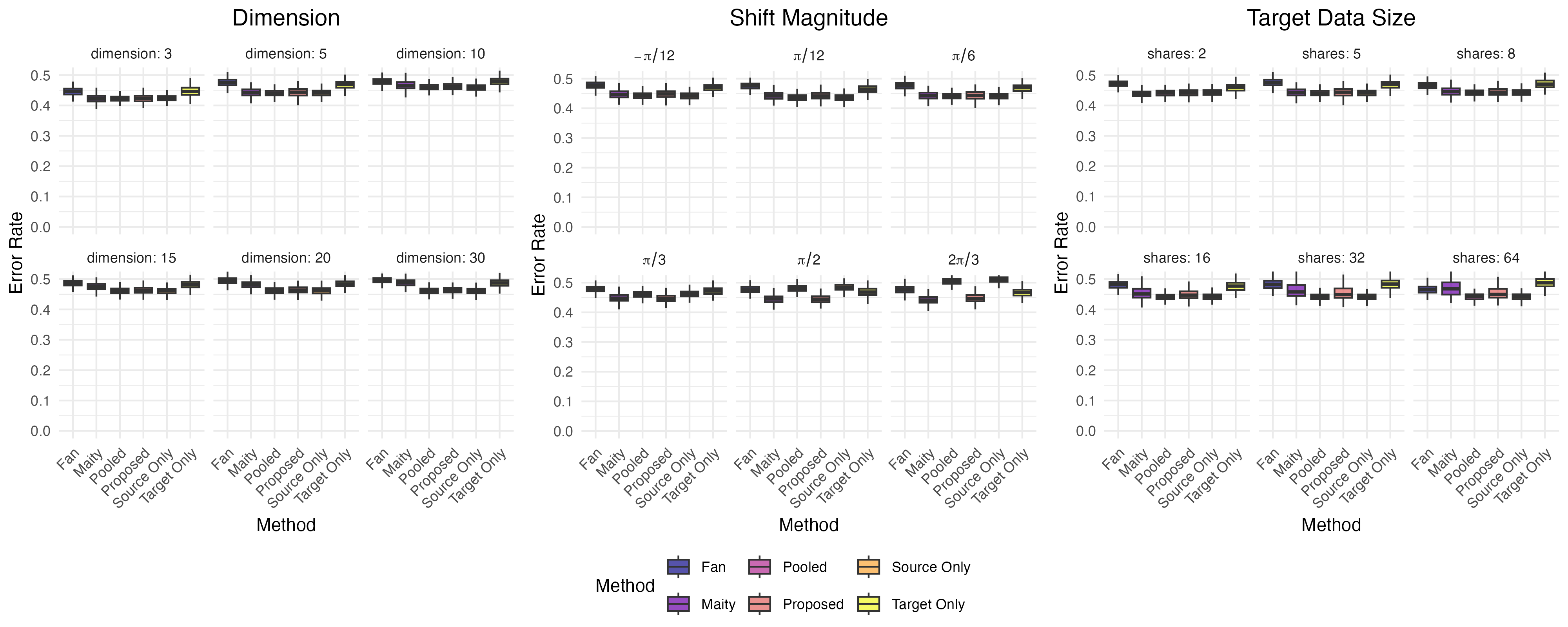}
  \caption{Misclassification rates under different dimension, shift magnitude, and data sizes for linear decision boundary with rotation transformation under linear regression function. Six methods as specified in the main document are compared.}
\end{figure}

\begin{figure}[p]
  \centering
  \includegraphics[ width=\textwidth]{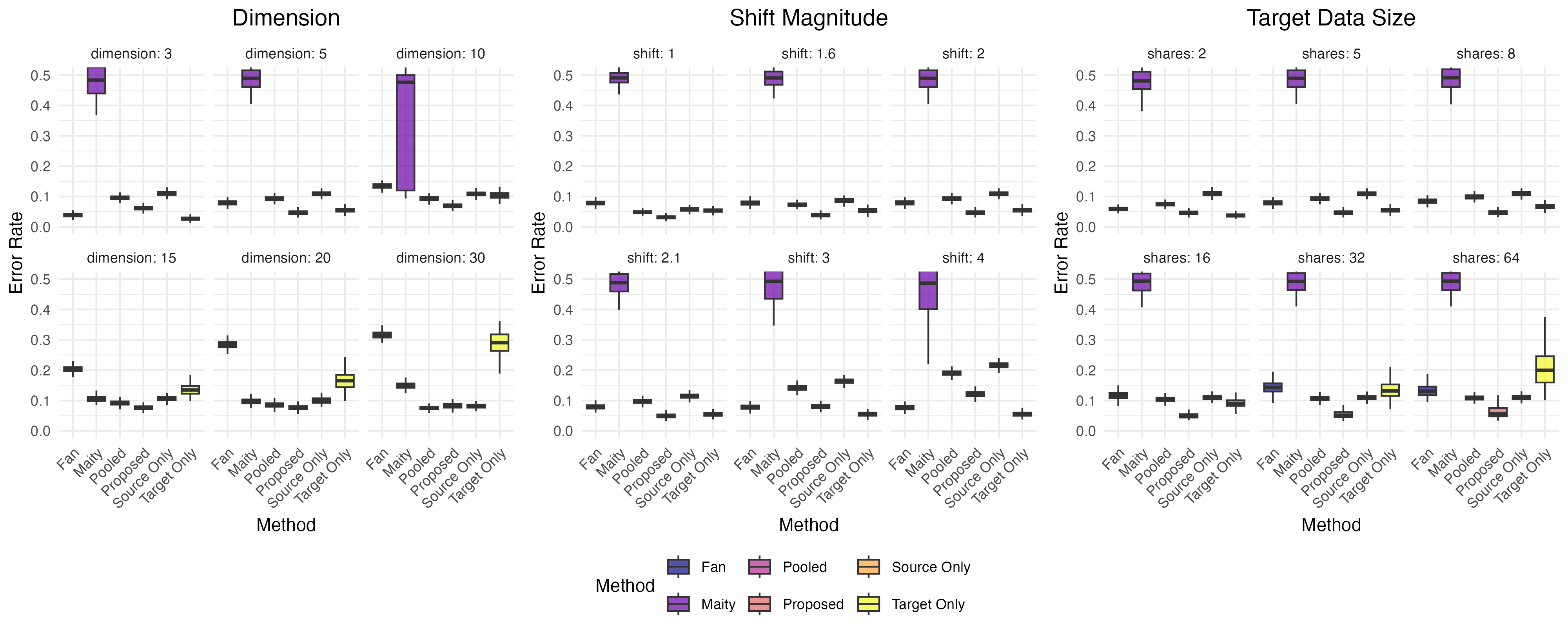}
  \caption{Misclassification rates under different dimension, shift magnitude, and data sizes for nonlinear decision boundary with translation transformation under deterministic regression function. Six methods as specified in the main document are compared.}
\end{figure}

\begin{figure}[p]
  \centering
  \includegraphics[ width=\textwidth]{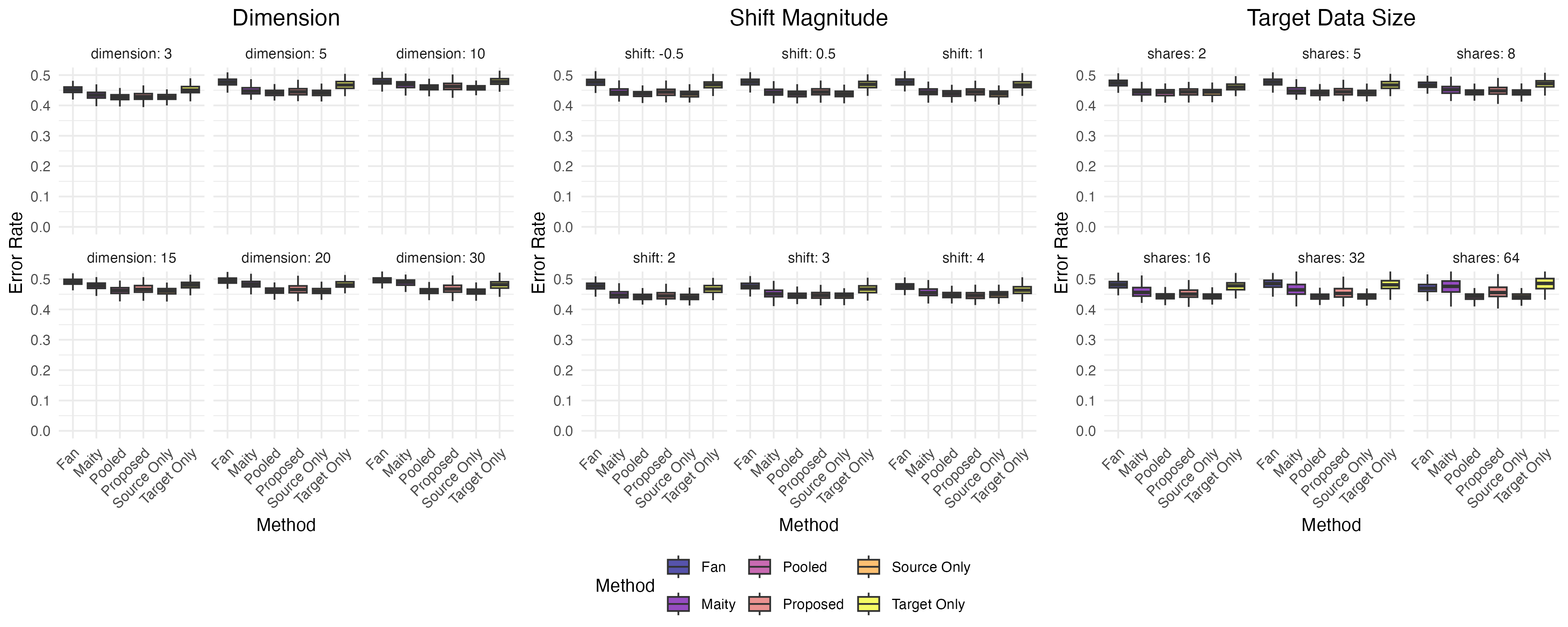}
  \caption{Misclassification rates under different dimension, shift magnitude, and data sizes for nonlinear decision boundary with translation transformation under linear regression function. Six methods as specified in the main document are compared.}
\end{figure}

\begin{figure}[p]
  \centering
  \includegraphics[ width=\textwidth]{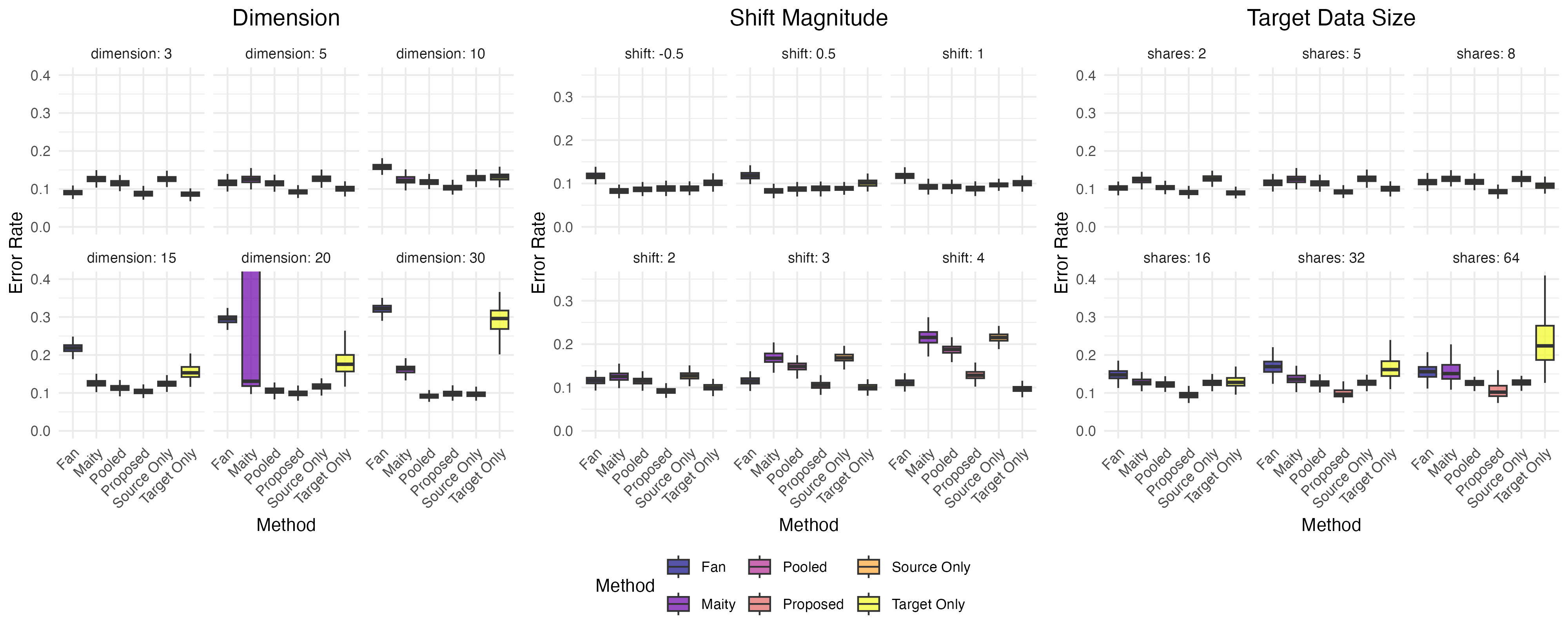}
  \caption{Misclassification rates under different dimension, shift magnitude, and data sizes for nonlinear decision boundary with translation transformation under logistic regression function. Six methods as specified in the main document are compared.}
\end{figure}

\begin{figure}[p]
  \centering
  \includegraphics[ width=\textwidth]{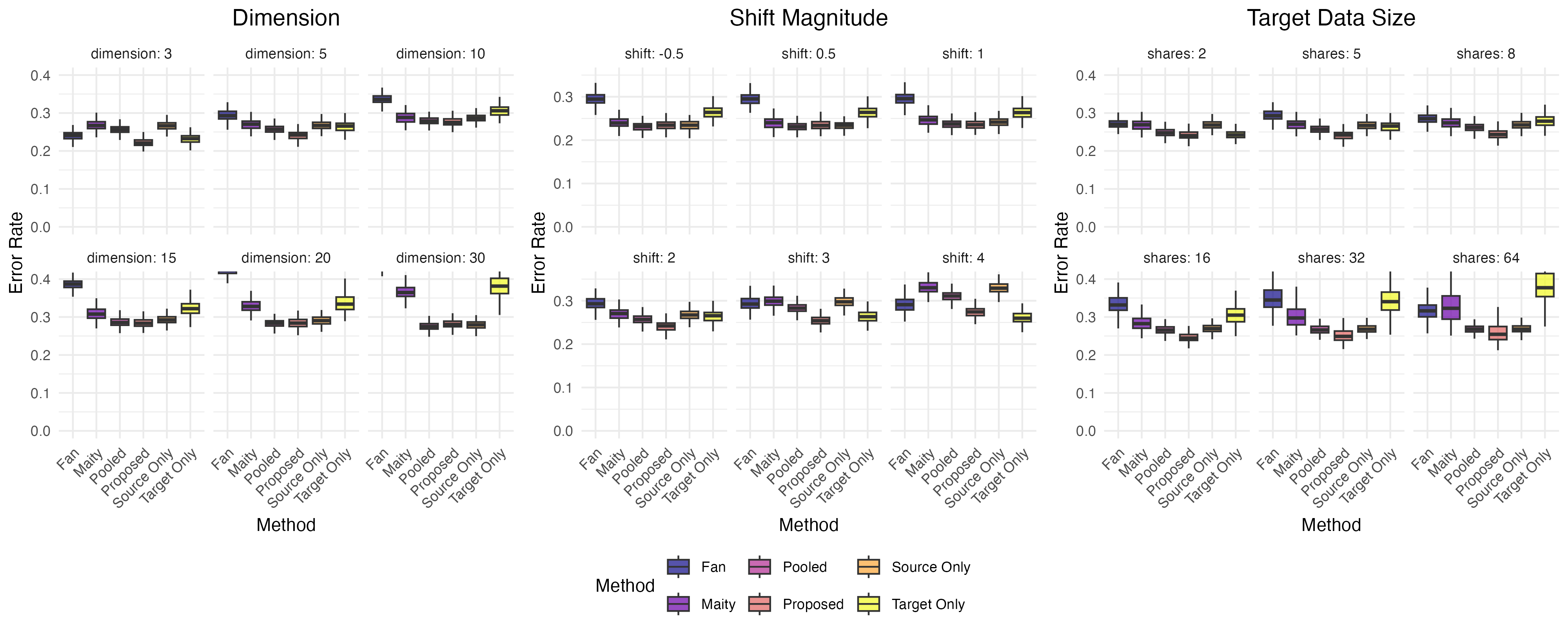}
  \caption{Misclassification rates under different dimension, shift magnitude, and data sizes for nonlinear decision boundary with translation transformation under truncated regression function. Six methods as specified in the main document are compared.}
\end{figure}

\begin{figure}[p]
  \centering
  \includegraphics[ width=\textwidth]{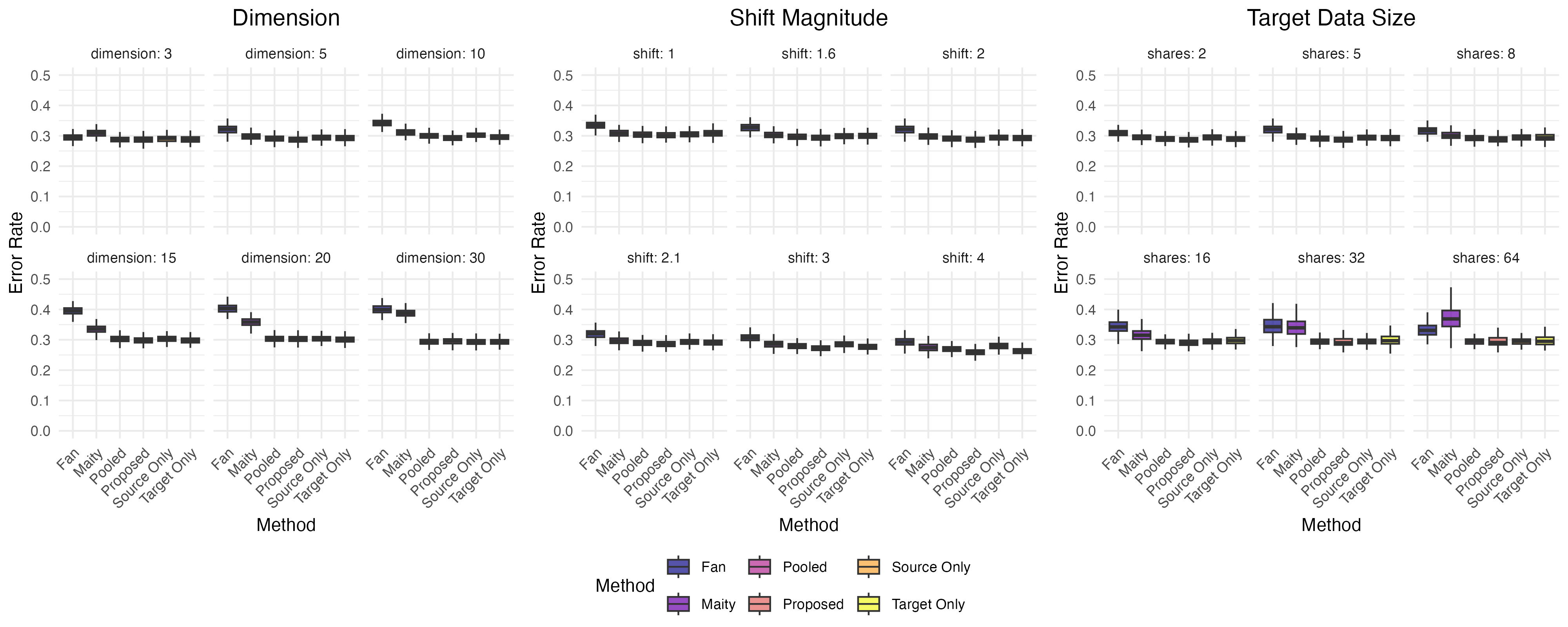}
  \caption{Misclassification rates under different dimension, shift magnitude, and data sizes for nonlinear decision boundary with translation transformation under truncated logistic regression function. Six methods as specified in the main document are compared.}
\end{figure}

\begin{figure}[p]
  \centering
  \includegraphics[ width=\textwidth]{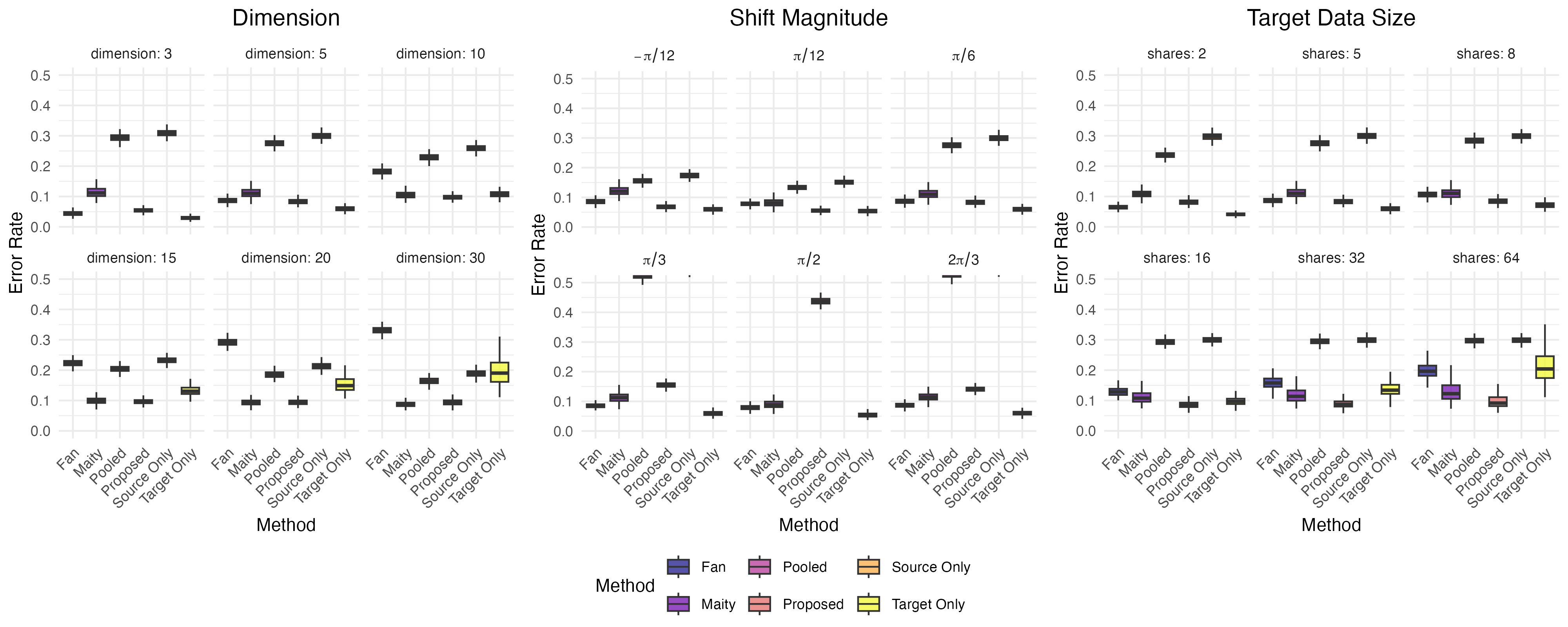}
  \caption{Misclassification rates under different dimension, shift magnitude, and data sizes for nonlinear decision boundary with rotation transformation under deterministic regression function. Six methods as specified in the main document are compared.}
\end{figure}

\begin{figure}[p]
  \centering
  \includegraphics[ width=\textwidth]{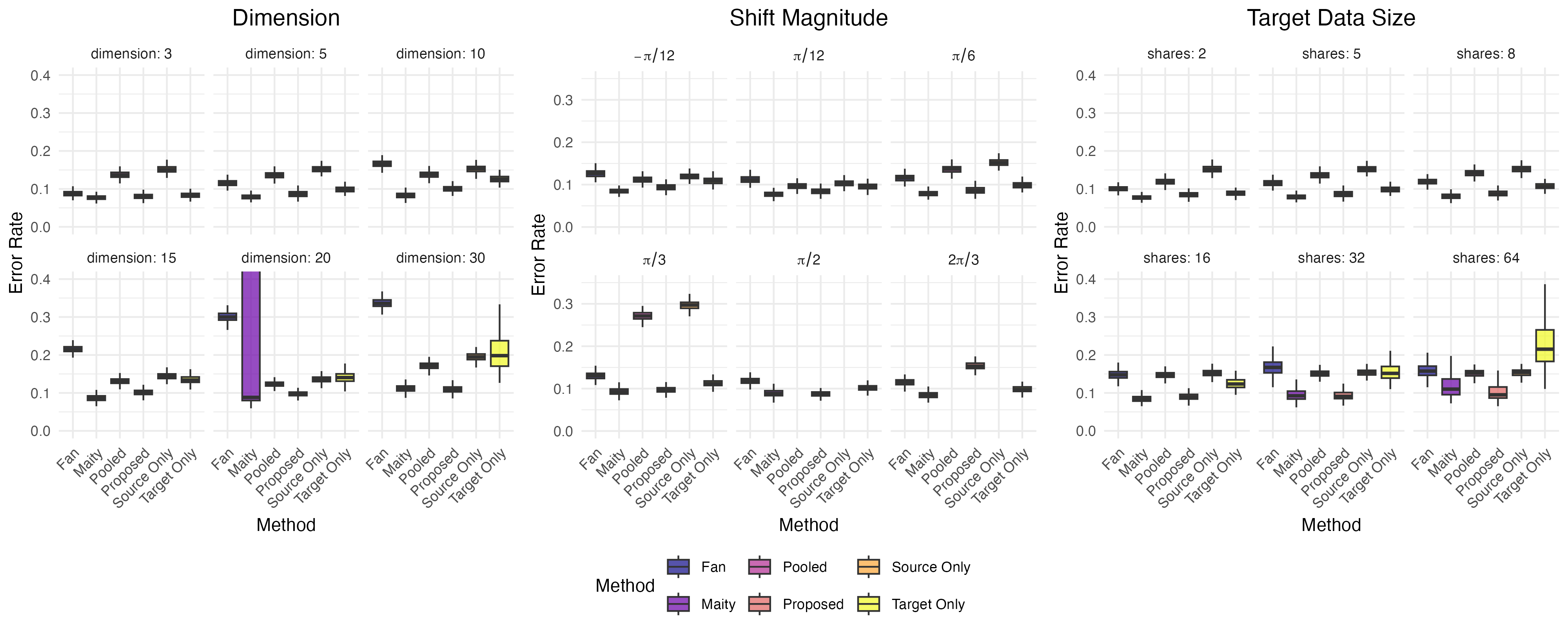}
  \caption{Misclassification rates under different dimension, shift magnitude, and data sizes for nonlinear decision boundary with rotation transformation under logistic regression function. Six methods as specified in the main document are compared.}
\end{figure}

\begin{figure}[p]
  \centering
  \includegraphics[ width=\textwidth]{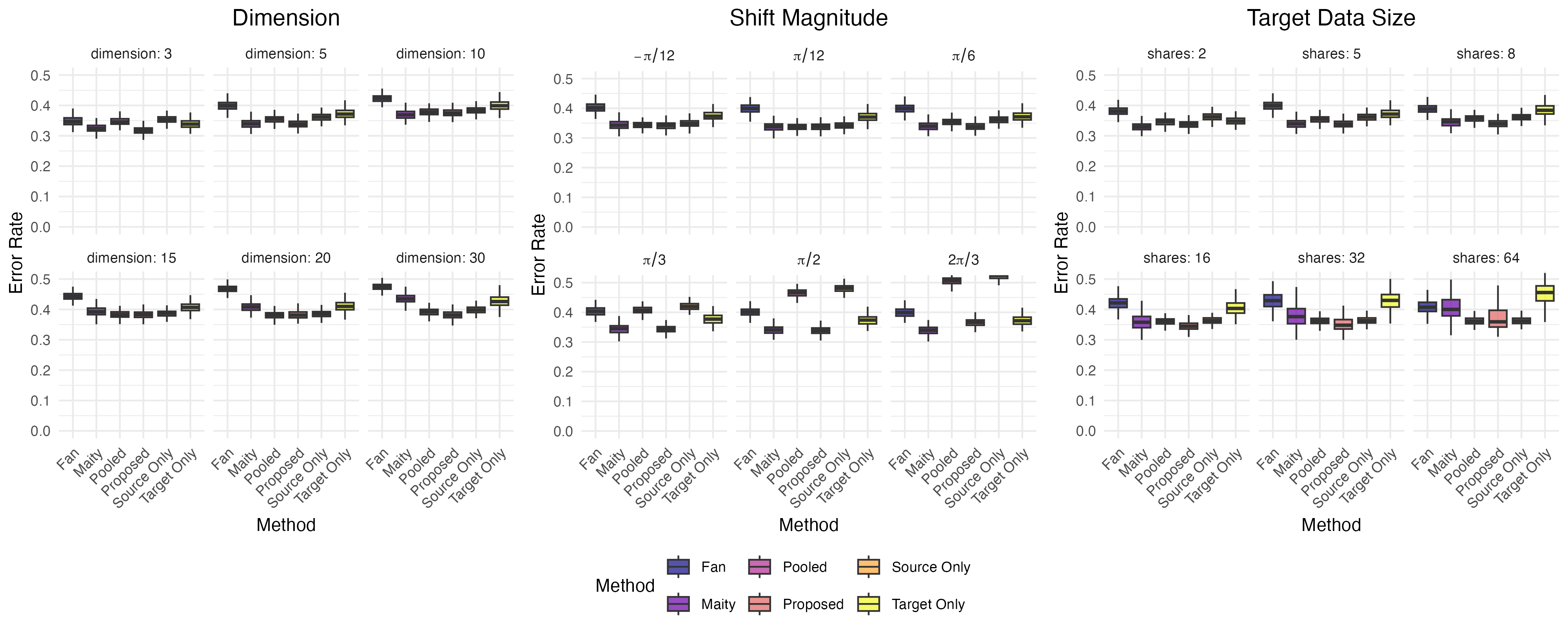}
  \caption{Misclassification rates under different dimension, shift magnitude, and data sizes for nonlinear decision boundary with rotation transformation under truncated regression function. Six methods as specified in the main document are compared.}
\end{figure}

\begin{figure}[p]
  \centering
  \includegraphics[ width=\textwidth]{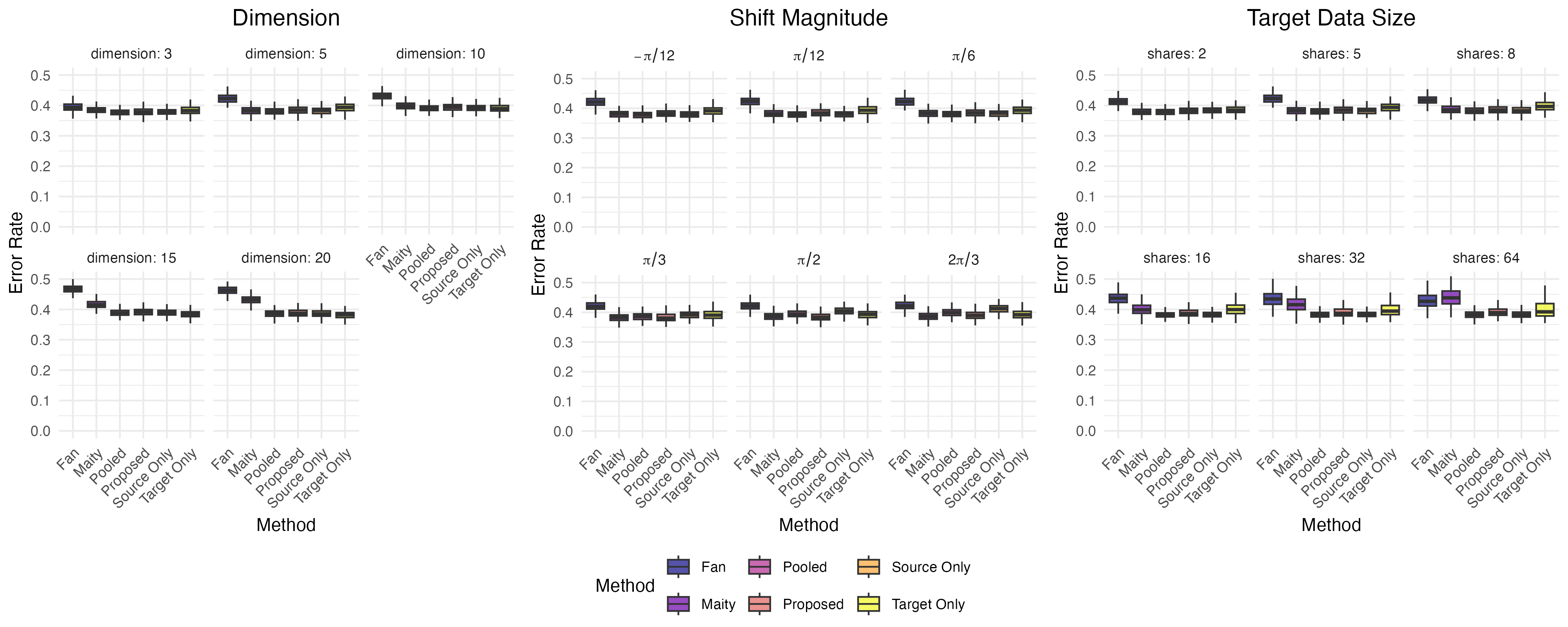}
  \caption{Misclassification rates under different dimension, shift magnitude, and data sizes for nonlinear decision boundary with rotation transformation under truncated logistic regression function. Six methods as specified in the main document are compared.}
\end{figure}

\clearpage

\bibliography{svmbib}
\end{document}